%% file: main.tex
\documentclass[letterpaper]{article} 
\usepackage{aaai23}  
\usepackage{times}  
\usepackage{helvet}  
\usepackage{courier}  
\usepackage[hyphens]{url}  
\usepackage{graphicx} 
\usepackage{natbib}  
\usepackage{caption} 
\usepackage{algorithm}
\usepackage{algorithmic}

\usepackage{dsfont}
\usepackage{enumitem}
\usepackage{tabularx}
\usepackage{booktabs}
\usepackage{subcaption}
\usepackage{caption}
\usepackage[export]{adjustbox}
\usepackage{wrapfig}
\usepackage{amsthm}
\usepackage{thmtools, thm-restate}
\declaretheorem[numberwithin=section]{thm}

\declaretheorem[sibling=thm]{lemma}

\declaretheorem[]{assumption}
\declaretheorem[]{definition}

\usepackage{xspace}
\DeclareRobustCommand{\eg}{e.g.,\@\xspace}
\DeclareRobustCommand{\ie}{i.e.,\@\xspace}
\DeclareRobustCommand{\wrt}{w.r.t.\@\xspace}
\usepackage{amsfonts}[mathscr]
\usepackage{amssymb}
\usepackage{amsmath}
\newcommand{\Aspace}{\mathcal{A}}
\newcommand{\Sspace}{\mathcal{S}}
\DeclareMathOperator{\EV}{\mathbb{E}}
\newcommand{\G}{\mathcal{G}}
\newcommand{\class}{\mathbb{M}}
\newcommand{\Pg}{P_\G}

\newcommand{\ind}{\mathbb{I}}
\newcommand{\mdp}{\mathcal{M}}
\DeclareMathOperator{\Pa}{Pa}
\newcommand{\one}{\mathds{1}}
\newcommand{\uni}{\mathbb{U}}
\newcommand{\V}{\mathcal{V}}

\title{Provably Efficient Causal Model-Based Reinforcement Learning for \\ Systematic Generalization}
\author {
    Mirco Mutti\equalcontrib \textsuperscript{\rm 1 2},
    Riccardo De Santi\equalcontrib \textsuperscript{\rm 3},
    Emanuele Rossi\textsuperscript{\rm 4 5},
    Juan Felipe Calderon\textsuperscript{\rm 1}, \\
    Michael Bronstein\textsuperscript{\rm 5 6},
    Marcello Restelli\textsuperscript{\rm 1}
}
\affiliations {
    \textsuperscript{\rm 1}Politecnico di Milano\\
    \textsuperscript{\rm 2}Universit\`a di Bologna\\
    \textsuperscript{\rm 3}ETH Zurich\\
    \textsuperscript{\rm 4}Imperial College London\\
    \textsuperscript{\rm 5}Twitter\\
    \textsuperscript{\rm 6}University of Oxford\\
    mirco.mutti@polimi.it, rdesanti@ethz.ch
}

\begin{document}

\maketitle

\begin{abstract}
In the sequential decision making setting, an agent aims to achieve \emph{systematic generalization} over a large, possibly infinite, set of environments. Such environments are modeled as discrete Markov decision processes with both states and actions represented through a feature vector. The underlying structure of the environments allows the transition dynamics to be factored into two components: one that is environment-specific and another that is shared.
Consider a set of environments that share the laws of motion as an example. In this setting, the agent can take a finite amount of \emph{reward-free} interactions from a subset of these environments. The agent then must be able to \emph{approximately} solve any planning task defined over any environment in the original set, relying on the above interactions only. Can we design a provably efficient algorithm that achieves this ambitious goal of systematic generalization?
In this paper, we give a partially positive answer to this question. First, we provide a tractable formulation of systematic generalization by employing a \emph{causal} viewpoint. Then, under specific structural assumptions, we provide a simple learning algorithm that guarantees any desired planning error up to an unavoidable sub-optimality term, while showcasing a polynomial sample complexity.
\end{abstract}

\section{Introduction}
Whereas recent breakthroughs have established Reinforcement Learning~\citep[RL,][]{sutton2018reinforcement} as a powerful tool to address a wide range of sequential decision making problems, the curse of generalization~\cite{kirk2021survey} is still a main limitation of commonly used techniques.
RL algorithms deployed on a given task are usually effective in discovering the correlation between an agent's behavior and the resulting performance from large amounts of labeled samples~\cite{jaksch2010near, lange2012batch}.
However, those algorithms are usually unable to discover basic cause-effect relations between the agent's behavior and the environment dynamics.
Crucially, the aforementioned correlations are oftentimes specific to the task at hand, and they are unlikely to be of any use for addressing different tasks or environments. Instead, some universal causal relations generalize over the environments, and once learned they can be exploited for solving any task. 
Let us consider as an illustrative example an agent interacting with a large set of physical environments. While each of these environments can have its specific dynamics, we expect the basic laws of motion to hold across the environments, as they encode general causal relations. 
Once they are learned, there is no need to discover them again from scratch when facing a new task, or an unseen environment. 
Even if the dynamics over these relations can change, such as moving underwater is different than moving in the air, or the gravity can change from planet to planet, the underlying causal structure still holds.
This knowledge alone often allows the agent to solve new tasks in unseen environments by taking a few, or even zero, interactions.
\begin{figure*}[ht!]
\centering \includegraphics[]{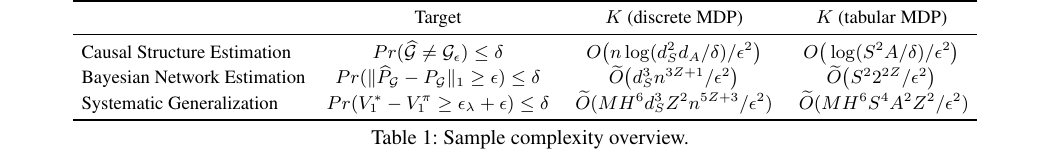}
\end{figure*}

We argue that we should pursue this kind of generalization in RL, which we call \emph{systematic generalization}, where learning universal causal relations from interactions with a few environments allows us to approximately solve any task in any other environment without further interactions.
Although this problem setting might seem overly ambitious or even far-fetched, in this paper we provide the first tractable formulation of systematic generalization, thanks to a set of structural assumptions that are motivated by a causal viewpoint.
The problem formulation is partially inspired by reward-free RL~\cite{jin2020reward}, in which the agent can take unlabelled interactions with an environment to learn a model that allows approximate planning for any reward function.
Here, we extend this formulation to a large, potentially infinite, set of reward-free environments, or a \emph{universe}, the agent can freely interact with.
We consider discrete environments, such that both their states and actions can be described through vectors of discrete features. Crucially, these environments share a common causal structure that explains a significant portion, but not all, of their transition dynamics.
Can we design a provably efficient algorithm that guarantees an arbitrarily small planning error for any possible task that can be defined over the set of environments, by taking reward-free interactions with a generative model?

In this paper, we provide a partially positive answer to this question by presenting a simple but principled causal model-based approach (see Figure~\ref{fig:pipeline}). This algorithm interacts with a finite subset of the universe to learn the causal structure underlying the set of environments in the form of a causal dependency graph $\G$. Then, the causal transition model, which encodes the dynamics that is common across the environment, is obtained by estimating the Bayesian network $P_{\G}$ over $\G$ from a mixture of the environments. Finally, the causal transition model is employed by a planning oracle to provide an approximately optimal policy for an unknown environment and a given reward function. We can show that this simple recipe, with a sample complexity that is polynomial in all the relevant quantities, allows achieving any desired planning error up to an unavoidable error term. The latter is inherent to the setting, which demands generalization over an infinite set of environments, and cannot be overcome without additional samples from the test environment.

The contributions of this paper include:
\begin{itemize}[noitemsep,topsep=0pt,parsep=0pt,partopsep=0pt,leftmargin=21pt]
    \item[(c1)] The first tractable formulation of the systematic generalization problem in RL, thanks to structural assumptions motivated by causal considerations (\S~\ref{sec:problem});
    \item[(c2)] A provably efficient algorithm to learn systematic generalization over an infinite set of environments (\S~\ref{sec:complexity_generalization});
    \item[(c3)] The sample complexity of estimating the causal structure underlying a discrete MDP (\S~\ref{sec:complexity_cuasal_discovery});
    \item[(c4)] The sample complexity of estimating the Bayesian network underlying a discrete MDP (\S~\ref{sec:complexity_bayesian_network});
    \item[(c5)] A brief numerical validation of the main results (\S~\ref{sec:experiments}).
\end{itemize}
On a technical level, (c3, c4) require the adaptation of known results in causal discovery~\cite{wadhwa2021sample} and Bayesian network estimation~\cite{dasgupta1997sample} to the specific MDP setting, which are then employed as building blocks to obtain the rate for systematic generalization (c2). See Table~1 for a summary of the main sample complexity results.

With this work we aim to connect several active research areas on model-based RL \cite{sutton2018reinforcement}, reward-free RL \cite{jin2020reward}, causal RL \cite{zhang2020invariant}, factored MDPs \cite{rosenberg2021oracle}, independence testing \cite{canonne2018testing}, experimental design \cite{ghassami2018budgeted} in a general framework where individual progresses can be enhanced beyond the sum of their parts.

\begin{figure*}[t]
    \centering
    \includegraphics[scale=0.22]{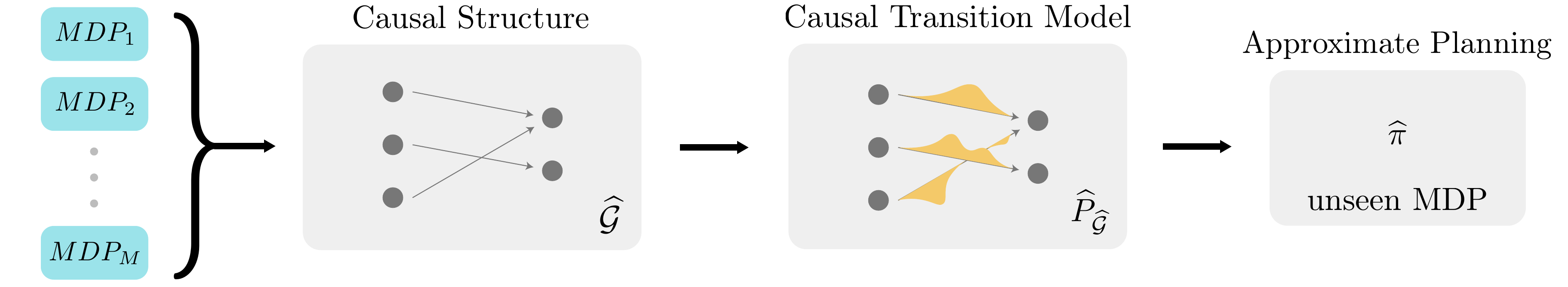}
    \caption{High-level illustration of the causal model-based approach to systematic generalization.}
    \label{fig:pipeline}
\end{figure*}

\section{Preliminaries}
\label{sec:preliminaries}
We start with some notions about graphs, causality, and Markov decision processes for later use. 
We denote a set of integers $\{1, \ldots, a \}$ as $[a]$, and the probability simplex over the space $\mathcal{A}$ as $\Delta_{\mathcal{A}}$. For a factored space $\mathcal{A} = \mathcal{A}_1 \times \ldots \times \mathcal{A}_a$ and a set of indices $Z \subseteq [a]$, which we call a \emph{scope}, we denote the scope operator as $\mathcal{A} [Z] := \bigotimes_{i \in Z} \mathcal{A}_i$, in which $\bigotimes$ is a cardinal product.
For any $A \in \mathcal{A}$, we denote with $A[Z]$ the vector $(A_i)_{i \in Z}$. For singletons we write $A[i]$ as a shorthand for $A[\{i\}]$.
Given  two probability measures $P$ and  $Q$ over a discrete space $\mathcal{A}$, their $L_1$-distance is  $\| P - Q \|_1 = \sum_{A \in \mathcal{A}} | P(A) - Q(A) |$, and their Kullback-Leibler (KL) divergence is $d_{KL} (P || Q) = \sum_{A \in \mathcal{A}} P(A) \log (P(A) / Q(A))$. 

\textbf{Graphs}~~~
We define a graph $\G$ as a pair $\G := (\V, E)$, where $\V$ is a set of nodes and $E \subseteq N\times N$ is a set of edges between them. We call $\G$ a \emph{directed graph} if all of its edges $E$ are directed (i.e., ordered pairs of nodes). We also define the in-degree of a node to be its number of incoming edges: $\mathrm{degree_{in}}(A)=|\{(B, A) : (B, A) \in E, \forall B\}|$. $\G$ is said to be a \emph{Directed Acyclic Graph} (DAG) if it is a directed graph without cycles. We call $\G$ a \emph{bipartite graph} if there exists a partition $X \cup Y = \V$ such that none of the nodes in $X$ and $Y$ are connected by an edge, \ie $E \cap (X \times X)  =  E \cap  (Y \times Y) = \emptyset$. For any subset of nodes $S \subset \V$, we define the \emph{subgraph} induced by $S$ as $\G [S] := (S, E[S])$, in which $E[S] = E \cap (S \times S)$. 
The \emph{skeleton} of a graph $\G$ is the undirected graph that is obtained from $\G$ by replacing all the directed edges in $E$ with undirected ones.
 Finally, the \emph{graph edit distance} between two graphs is the minimum number of graph edits (addition or deletion of either a node or an edge) necessary to transform one graph into the other.

\textbf{Causal Graphs and Bayesian Networks}~~~
For a set $\mathcal{X}$ of random variables, we represent the causal structure over $\mathcal{X}$ with a DAG $\G_{\mathcal{X}} = (\mathcal{X}, E)$,\footnote{We will omit the subscript $\mathcal{X}$ whenever clear from the context.} 
which we call the \emph{causal graph} of $\mathcal{X}$.
For each pair of variables $A, B \in \mathcal{X}$, a directed edge $(A, B) \in \G_{\mathcal{X}}$ denotes that $B$ is conditionally dependent on $A$.
For every variable $A \in \mathcal{X}$, we denote as $\Pa(A)$ the \emph{causal parents} of $A$, \ie the set of all the variables $B \in \mathcal{X}$ on which $A$ is conditionally dependent, $(B, A) \in \G_\mathcal{X}$. 
A Bayesian network~\cite{dean1989model} over the set $\mathcal{X}$ is defined as $\mathcal{N} := (\G_{\mathcal{X}}, P)$, where $\G_\mathcal{X}$ specifies the \emph{structure} of the network, \ie the dependencies between the variables in $\mathcal{X}$, and the distribution $P : \mathcal{X} \to \Delta_{\mathcal{X}}$ specifies the conditional probabilities of the variables in $\mathcal{X}$, such that
$
    P(\mathcal{X}) = \prod_{X_i \in \mathcal{X}} P_i (X_i | \Pa (X_i) ).
$

\textbf{Markov Decision Processes}~~~
A \emph{tabular} episodic Markov Decision Process~\citep[MDP,][]{puterman2014markov} is defined as $\mdp := (\Sspace, \Aspace, P, H, r)$, where $\Sspace$ is a set of $|\Sspace| = S$ states, $\Aspace$ is a set of $|\Aspace| = A$ actions, $P$ is a transition model such that $P(s'|s, a)$ gives the conditional probability of the next state $s'$ having taken action $a$ in state $s$, $H$ is the episode horizon, $r : \Sspace \times \Aspace \to [0, 1]$ is a deterministic reward function.

The strategy of an agent interacting with $\mdp$ is represented by a non-stationary, stochastic \emph{policy}, a collection of functions $(\pi_h : \Sspace \to \Delta_{\Aspace})_{h \in [H]}$ where $\pi_h (a|s)$ denotes the conditional probability of taking action $a$ in state $s$ at step $h$. The \emph{value function} $V^\pi_h : \Sspace \to \mathbb{R}$ associated to $\pi$ is defined as the expected sum of the rewards that will be collected, under the policy $\pi$, starting from $s$ at step $h$, \ie
\begin{equation*}
    V^\pi_h (s) := \EV_{\pi} \bigg[ \sum_{h' = h}^H r (s_{h'}, a_{h'}) \ \Big| \ s_{h} = s \bigg].
\end{equation*}
For later convenience, we further define $P V^\pi_{h + 1} (s, a) := \EV_{s' \sim P(\cdot | s, a)} [ V_{h + 1}^{\pi} (s')]$ and $V_1^\pi := \EV_{s \sim P} [ V^\pi_1 (s)]$. We will write $V^\pi_{\mdp, r}$ to denote $V^\pi_1$ in the MDP $\mdp$ with reward function $r$ (if not obvious from the context).  
For an MDP $\mdp$ with finite states, actions, and horizon, there always exists an \emph{optimal policy} $\pi^*$ that gives the value $V^*_h (s) = \sup_{\pi} V^\pi_h (s)$ for every $s, a, h$. The goal of the agent is to find a policy $\pi$ that is $\epsilon$-close to the optimal one, \ie $ V^*_1 - V^\pi_1 \leq \epsilon.$

Finally, we define a \emph{discrete} Markov decision process as $\mdp := ((\Sspace, d_S, n), (\Aspace, d_A, n), P, H, r)$, where $\Sspace, \Aspace, P, H, r$ are specified as before, and where the states and actions spaces admit additional structure, such that every $s \in \Sspace$ can be represented through a $d_S$-dimensional  vector of discrete features taking value in $[n]$, and every $a \in \Aspace$ can be represented through a $d_A$-dimensional vector of discrete features taking value in $[n]$. Note that any tabular MDP can be formulated under this alternative formalism through one-hot encoding by taking $n = 2$, $d_S = S$, and $d_A = A$.

\section{Problem Formulation}
\label{sec:problem}
In our setting, a learning agent aims to master a large, potentially infinite, set $\uni$ of environments modeled as discrete MDPs without rewards that we call a \emph{universe}
\begin{equation*}
    \uni := \big\{ \mdp_i = ((\Sspace, d_S, n), (\Aspace, d_A, n), P_i, \mu) \big\}_{i = 1}^\infty.
\end{equation*}
The agent can draw a finite amount of experience by interacting with the MDPs in $\uni$. From these interactions alone, the agent aims to acquire sufficient knowledge to approximately solve any task that can be specified over the universe $\uni$. A \emph{task} is defined as any pairing of an MDP $\mdp \in \uni$ and a reward function $r$, whereas \emph{solving it} refers to providing a slightly sub-optimal policy via planning, \ie without taking additional interactions. We call this problem \emph{systematic generalization}, which we can formalize as follows.
\begin{definition}[Systematic Generalization]
    \label{def:systematic_generalization}
    For any unknown MDP $\mdp \in \uni$ and any given reward function $r: \Sspace \times \Aspace \to [0, 1]$, the {\em systematic generalization} problem requires the agent to provide a policy $\pi$, such that
    $
        V^*_{\mdp, r}  - V^{\pi}_{\mdp, r} \leq \epsilon
    $
    up to any desired sub-optimality $\epsilon > 0$.
\end{definition}
Since the set $\uni$ is infinite, we clearly require additional structure to make the problem feasible. On the one hand, the state space $(\Sspace, d_S, n)$, action space $(\Aspace, d_A, n)$, and initial state distribution $\mu$ are shared across $\mdp \in \uni$. The transition dynamics $P_i$ is instead specific to each MDP $\mdp_i \in \uni$. However, we assume the presence of a {\em common causal structure} that underlies the transition dynamics of the universe, and relates the single transition models $P_i$.

\subsection{Causal Structure of the Transition Dynamics}
The transition dynamics of a discrete MDP gives the conditional probability of next state features $s'$ given the current state-action features $(s, a)$. To ease the notation, from now on we will denote the state-action features with a random vector $X = (X_i)_{i \in [d_S + d_A]}$, in which each $X_i$ is supported in $[n]$, and the next state features with a random vector $Y = (Y_i)_{i \in [d_S]}$, in which each $Y_i$ is supported in $[n]$.

For each environment $\mdp_i \in \uni$, the conditional dependencies between the next state features $Y$ and the current state-action features $X$ are represented through a bipartite dependency graph $\G_i$, such that $(X[z], Y[j]) \in \G_i$ if and only if $Y[j]$ is conditionally dependent on $X[z]$.
Clearly, each environment can display its own dependencies, but we assume there is a set of dependencies that represent general causal relationships between the features, and that appear in any $\mdp_i \in \uni$. In particular, we call the intersection $\G := \cap_{i = 0}^\infty \G_i$ the \emph{causal structure} of $\uni$, which is the set of conditional dependencies that are common across the universe. In Figure~\ref{fig:causal_transition_model}, we show an illustration of such a causal structure.
Since it represents universal causal relationships, the causal structure $\G$ is time-consistent, \ie $\G^{(h)} = \G^{(1)}$ for any step $h \in [H]$, and we further assume that $\G$ is sparse, which means that the number of features $X[z]$ on which a feature $Y[j]$ is dependent on is bounded from above.
\begin{assumption}[$Z$-sparseness]
\label{ass:sparseness}
    The causal structure $\G$ is {\em $Z$-sparse} if  $\max_{j \in [d_S]} \mathrm{degree_{in}}(Y[j]) \leq Z$.
\end{assumption}

Given a causal structure $\G$, without loosing generality\footnote{Note that one can always take $P_\G (Y | Z) = 1, \forall (X, Y)$.} we can express each transition model $P_i$ as
$
    P_i (Y | X) = \Pg (Y | X) F_i (Y | X),
$
in which $\Pg$ is the Bayesian network over the causal structure $\G$, whereas $F_i$ includes environment-specific factors.\footnote{The parameters in $F_i$ are numerical values such that $P_i$ remains a well-defined probability measure.} Since it represents the conditional probabilities due to universal causal relations in $\uni$, we call $P_\G$ the \emph{causal transition model} of $\uni$. Thanks to the structure $\G$, $P_\G$ can be further factored as
\begin{equation}
    \label{eq:factorization}
    \Pg (Y | X) = \prod_{j = 1}^{d_S} P_j (Y[j] | X[Z_j] ),
\end{equation}
where the scopes $Z_j$ are the the causal parents of $Y[j]$, \ie $(X[z], Y[j]) \in \G, \forall z \in Z_j$. In Figure~\ref{fig:bayesian_network}, we show an illustration of the causal transition model and its factorization.
Similarly to the underlying structure $\G$, the causal transition model $P_\G$ is also time-consistent, \ie $P_\G^{(h)} = P_{\G}^{(1)}$ for any step $h \in [H]$. In this work, we assume that the causal transition model is non-vacuous and that it explains a significant part of the transition dynamics of $\mdp_i \in \uni$.
\begin{assumption}[$\lambda$-sufficiency]
    \label{ass:sufficiency}
    Let $\lambda \in [0, 1]$ be a constant. The causal transition model $P_{\G}$ is {\em causally $\lambda$-sufficient} if 
    $
        \sup_{X} \| P_{\G} (\cdot | X) - P_i (\cdot | X) \|_1 \leq \lambda, \ \forall P_i \in \mdp_i \in \uni.
    $
\end{assumption}
The  parameter $\lambda$ controls the amount of the transition dynamics that is due to the universal causal relations $\G$ ($\lambda = 0$ means that $P_\G$ is sufficient to explain the transition dynamics of any $\mdp_i \in \uni$, whereas $\lambda = 1$ implies no shared structure). In this paper, we argue that learning the causal transition model $P_\G$ is a good target for systematic generalization and we provide theoretical support for this claim in \S~\ref{sec:complexity}.
\begin{figure}[t!]
    \centering
    \includegraphics[width=0.85\linewidth]{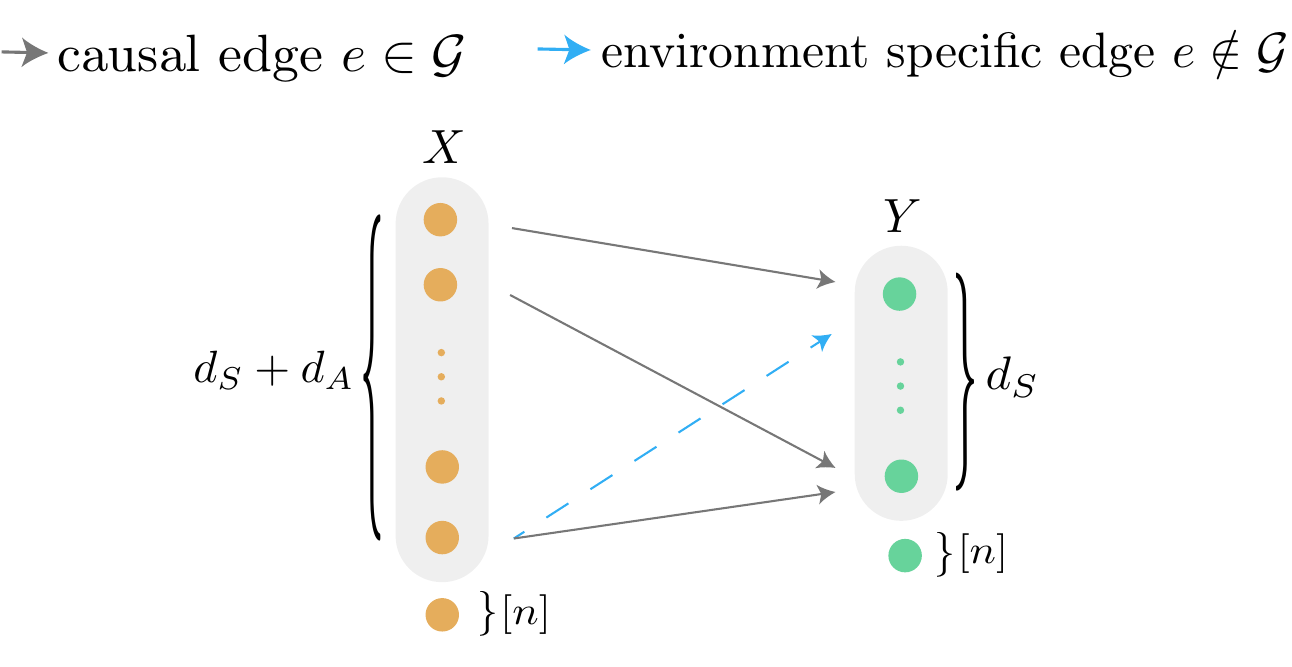}
    \caption{Causal structure $\G$ of $\uni$.}
    \label{fig:causal_transition_model}
\end{figure}

\subsection{A Class of Training Environments}

Even if the universe $\uni$ admits the structure that we presented in the last section, it is still an infinite set. Instead, the agent can only interact with a finite subset of discrete MDPs
\begin{equation*}
    \class := \{ \mdp_i = ((\Sspace, d_S, n), (\Aspace, d_A, n), P_i, \mu)\}_{i = 1}^M \subset \uni,
\end{equation*}
which we call a \emph{class} of size $M$. Crucially, the causal structure $\G$ is a property of the full set $\uni$, and if we aim to infer it from interactions with a finite class $\class$, we have to assume that $\class$ is informative enough on the structure of $\uni$.
\begin{assumption}[Diversity]
    \label{ass:diversity}
    Let $\class \subset \uni$ be class of size $M$. We say that $\class$ is {\em causally diverse} if
    $\G = \cap_{i = 1}^{M} \G_i = \cap_{i = 1}^{\infty} \G_i.$\footnote{W.l.o.g., we assume that the indices $i \in [M]$ refers to the $\mdp_i \in \class$, and $i \in (M, \infty)$ to the $\mdp_i \in \uni \setminus \class$.}
\end{assumption}
Analogously, if we aim to infer the causal transition model $P_\G$ from interactions with the transition models $P_i$ of the single MDPs $\mdp_i \in \class$, we have to assume that $\class$ is balanced in terms of the conditional probabilities displayed by its components, so that the factors that do not represent universal causal relations even out while learning.
\begin{assumption}[Evenness]
    \label{ass:evenness}
    Let $\class \subset \uni$ a class of size $M$. We say that $\class$ is {\em causally even} if \footnote{We denote by $\mathcal{U}_{[M]}$ the uniform distribution over $[M]$.}
    $$
         \EV_{i \sim \mathcal{U}_{[M]}} \big[ F_i (Y[j] | X) \big] = 1, \ \ \forall j \in [d_S].
    $$
\end{assumption}
In this paper we assume that $\class$ is \emph{diverse} and \emph{even} by design, while we leave as future work the problem of selecting such a class from active interactions with $\uni$, which would add to our formulation flavors of active learning and experimental design~\cite{2014buhlmann, KocaogluExp2017, ghassami2018budgeted}.
\begin{figure}[t!]
    \centering
    \includegraphics[width=0.69\linewidth]{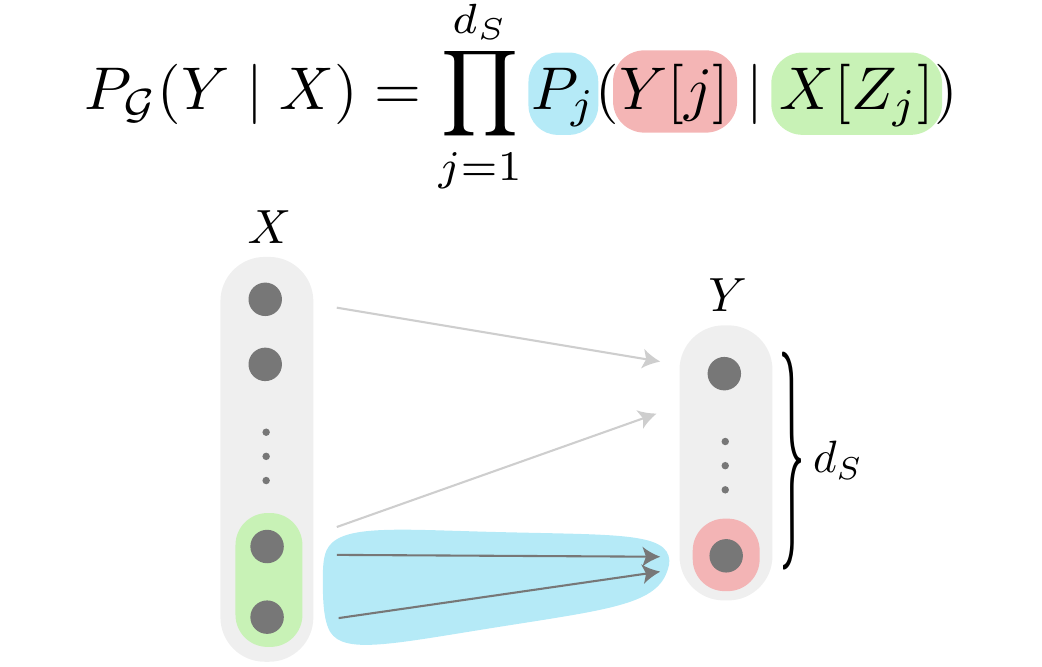}
    \caption{Causal transition model $P_\G$ of $\uni$.}
    \label{fig:bayesian_network}
\end{figure}

\subsection{Learning Systematic Generalization}
\label{sec:causal_perspective}
Before addressing the sample complexity of systematic generalization, it is worth considering the kind of interactions that we need in order to learn the causal transition model $P_\G$ and its underlying causal structure $\G$. Especially, thanks to the peculiar configuration of the causal structure $\G$, \ie a bipartite graph in which the edges are necessarily directed from the state-action features $X$ to the next state features $Y$, as a causation can only happen from the past to the future, learning the skeleton of $\G$ is equivalent to learning its full structure. Crucially, learning the skeleton of a causal graph does not need specific interventions, as it can be done from observational data alone~\cite{2014buhlmann}.
\begin{restatable}[]{proposition}{observationalDataSufficiency}
    \label{thr:observational_data_sufficiency}
    The causal structure $\G$ of $\;\uni$ can be identified from purely observational data.
\end{restatable}
In this paper, we will consider the online learning setting with a \emph{generative model} for estimating $\G$ and $P_\G$ from sampled interactions with a class $\class$ of size $M$. A generative model allows the agent to set the state of an MDP before sampling a transition, instead of drawing sequential interactions from the process.
Finally, analogous results to what we obtain here can apply to the offline setting as well, in addition to convenient coverage assumptions on the dataset.

\section{Sample Complexity Analysis}
\label{sec:complexity}
We provide a sample complexity analysis of the problem, which stands as a core contribution of this paper along with the problem formulation itself (\S~\ref{sec:problem}).
First, we consider the sample complexity of systematic generalization (\S~\ref{sec:complexity_generalization}). Then, we provide ancillary results on the estimation of the causal structure (\S~\ref{sec:complexity_cuasal_discovery}) and the Bayesian network (\S~\ref{sec:complexity_bayesian_network}) of an MDP, which can be of independent interest.

\subsection{Sample Complexity of Systematic Generalization with a Generative Model}
\label{sec:complexity_generalization}
We have access to a class $\class$ of discrete MDPs within a universe $\uni$, from which we draw interactions with a generative model $P(X)$. We aim to solve the systematic generalization problem as described in Definition~\ref{def:systematic_generalization}. This problem requires to provide, for any combination of an (unknown) MDP $\mdp \in \uni$, and a given reward function $r$, a planning policy $\widehat{\pi}$ such that
$ V^*_{\mdp, r} - V^{\widehat{\pi}}_{\mdp, r} \leq \epsilon.$
Especially, can we design an algorithm that guarantees this requirement with high probability by taking a number of samples $K$ that is polynomial in $\epsilon$ and the relevant parameters of $\class$? Here we give a partially positive answer to this question, by providing a simple but provably efficient algorithm that guarantees systematic generalization over $\uni$ up to an unavoidable sub-optimality term $\epsilon_\lambda$ that we will later specify.

The algorithm implements a model-based approach into two separated components. The first component is the procedure that actually interacts with the class $\class$ to obtain a principled estimation $\widehat{P}_{\widehat{\G}}$ of the causal transition model $P_\G$ of $\uni$. The second, is a planning oracle that takes as input a reward function $r$ and the estimated causal transition model, and returns an optimal policy $\widehat{\pi}$ operating on $\widehat{P}_{\widehat{\G}}$ as an approximation of the transition model $P_i$ of the true MDP $\mdp_i$.\footnote{The planning oracle can be substituted with a principled approximate planning solver~\citep[see][Section 3.3]{jin2020reward}.}

First, we provide the sample complexity of the causal transition model estimation (Algorithm~\ref{alg:learning_ctm}), which in turn is based on repeated causal structure estimations (Algorithm~\ref{alg:learning_g}) to obtain $\widehat{\G}$, and an estimation procedure of the Bayesian network over $\widehat{\G}$ (Algorithm~\ref{alg:learning_p}) to obtain $\widehat{P}_{\widehat{\G}}$.
\begin{restatable}[]{lemma}{sampleComplexityCTM}
\label{thr:sample_complexity_ctm}
    Let $\class = \{ \mdp_i \}_{i = 1}^M$ be a class of $M$ discrete MDPs, let $\delta \in (0, 1)$, $\epsilon > 0$. The Algorithm~\ref{alg:learning_ctm} returns a causal transition model $\widehat{P}_{\widehat{\G}}$ such that $ Pr ( \| \widehat{P}_{\widehat{\G}} - P_\G \|_1 \geq \epsilon ) \leq \delta $ with a sample complexity
    \begin{equation*}
        K = O \Big( M d_S^3 Z^2 n^{3Z + 1} \log \big( \tfrac{ 4 M d_S^2 d_A n^Z }{ \delta } \big) \Big/ \epsilon^2 \Big).
    \end{equation*}
\end{restatable}
An analogous result can be derived for tabular MDPs.
\begin{restatable}[]{lemma}{sampleComplexityCTMtabular}
\label{thr:sample_complexity_ctm_tabular}
    Let $\class = \{ \mdp_i \}_{i = 1}^M$ be a class of $M$ tabular MDPs. The result of Lemma~\ref{thr:sample_complexity_ctm} reduces to
    \begin{equation*}
        K = O \Big( M S^2 Z^2 2^{2Z} \log \big( \tfrac{ 4 M S^2 A 2^Z }{ \delta} \big) \Big/ \epsilon^2 \Big).
    \end{equation*}
\end{restatable}

Having established the sample complexity of the causal transition model estimation, we can now show how the learned model $\widehat{P}_{\widehat{\G}}$ allows us to approximately solve, via a planning oracle, any task defined by a combination of a latent MDP $\mdp_i \in \uni$ and a given reward function $r$. 
\begin{algorithm}[t!]
    \caption{Causal Transition Model Estimation}
    \label{alg:learning_ctm}
    \begin{algorithmic}[t]
        \STATE \textbf{Input}: class of MDPs $\class$, error $\epsilon$, confidence $\delta$
        \STATE let $K' = C' \big( d_S^2 Z^2 n  \log (2 M d_S^2 d_A / \delta ) \big/ \epsilon^2 \big)$
        \STATE set the generative model $P(X) = \mathcal{U}_{X}$
        \FOR{$i = 1, \ldots, M$}
            \STATE let $P_i (Y|X)$ the transition model of $\mdp_i \in \class$
            \STATE $\widehat{\G}_i \leftarrow$ \emph{Causal Structure Estimation} $(P_i, P(X), K')$
        \ENDFOR
        \STATE let $\widehat{\G} = \cap_{i = 1}^M \widehat{\G}_i$
        \STATE let $K'' = C'' \big( d_S^3  n^{3Z + 1} \log (4 d_S n^Z / \delta ) \big/ \epsilon^2 \big)$
        \STATE let $P_\class (Y | X)$ be the mixture $\frac{1}{M} \sum_{i = 1}^M P_i (Y | X)$
        \STATE $\widehat{P}_{\widehat{\G}} \leftarrow$ \emph{Bayesian Network Estimation} $(P_\class, \widehat{\G}, K'')$
        \STATE \textbf{Output}: causal transition model $\widehat{P}_{\widehat{\G}}$
    \end{algorithmic}
\end{algorithm}

To provide this result in the discrete MDP setting, we have to further assume that the transition dynamics $P_i$ of the target MDP $\mdp_i$ admits factorization analogous to~\eqref{eq:factorization}, such that we can write $P_i (Y | X) = \prod_{j = 1}^{d_S} P_{i, j} (Y[j] | X[Z_j^{'}])$, where the scopes $Z_j^{'}$ are given by the environment causal structure $\G_i$, which we assume to be $2Z$-sparse (Assumption~\ref{ass:sparseness}).
\begin{restatable}[]{theorem}{sampleComplexity}
\label{thr:sample_complexity}
    Let $\delta \in (0, 1)$ and $\epsilon > 0$. For an unknown discrete MDP $\mdp \in \uni$, and a given reward function $r$, a planning oracle operating on the causal transition model $\widehat{P}_{\widehat{\G}}$ as an approximation of $\mdp$ returns a policy $\widehat{\pi}$ such that
    $
        Pr \big( V^*_{\mdp_i, r} - V_{\mdp_i, r}^{\widehat{\pi}} \geq \epsilon_\lambda + \epsilon \big) \leq \delta,
    $
    where $\epsilon_\lambda = 2 \lambda H^3 d_S n^{2Z + 1}$, and $\widehat{P}_{\widehat{\G}}$ is obtained from Algorithm~\ref{alg:learning_ctm} with $\delta' = \delta$ and $\epsilon' = \epsilon / 2 H^3 n^{Z + 1}$.
\end{restatable}
Without the additional factorization of the environment-specific transition model, the result of Theorem~\ref{thr:sample_complexity} reduces to the analogous for the tabular MDP setting.
\begin{restatable}[]{corollary}{sampleComplexitytabular}
    \label{thr:sample_complexity_tabular}
    Let $\mdp$ a tabular MDP, the result of Theorem~\ref{thr:sample_complexity} holds with $\epsilon_\lambda = 2 \lambda S A H^3 $, $\epsilon' = \epsilon / 2 S A H^3$.
\end{restatable}

Theorem~\ref{thr:sample_complexity} and Corollary~\ref{thr:sample_complexity_tabular} establish the sample complexity of systematic generalization through Lemma~\ref{thr:sample_complexity_ctm} and Lemma~\ref{thr:sample_complexity_ctm_tabular} respectively. For the discrete MDP setting, we have that $\widetilde{O} (M H^6 d_S^3 Z^2 n^{5Z + 3} )$ samples are required, which reduces to $\widetilde{O} (M H^6 S^4 A^2 Z^2 )$ in the tabular setting. Unfortunately, we are only able to obtain systematic generalization up to an unavoidable sub-optimality term $\epsilon_\lambda$. This error term is related to the $\lambda$-sufficiency of the causal transition model (Assumption~\ref{ass:sufficiency}), and it accounts for the fact that $P_\G$ cannot fully explain the transition dynamics of each $\mdp \in \uni$, even when it is estimated exactly. This is inherent to the ambitious problem setting, and can be only overcome with additional interactions with the test MDP $\mdp$.

\subsection{Sample Complexity of Learning the Causal Structure of a Discrete MDP}
\label{sec:complexity_cuasal_discovery}
As a byproduct of the main result in Theorem~\ref{thr:sample_complexity}, we can provide a sample complexity result for the problem of learning the causal structure $\G$ underlying a discrete MDP $\mdp$ with a generative model.
We believe that this problem can be of independent interest, mainly in consideration of previous work on causal discovery of general stochastic processes~\citep[\eg][]{wadhwa2021sample}, for which we refine known results to account for the structure of an MDP, which allows for a tighter analysis of the sample complexity.
\begin{algorithm}[t!]
    \caption{MDP Causal Structure Estimation}
    \label{alg:learning_g}
    \begin{algorithmic}[t]
        \STATE \textbf{Input}: sampling model $P(Y|X)$, generative model $P(X)$, batch parameter $K$
        \STATE draw $(x_k, y_k)_{k = 1}^{K} \overset{\text{iid}}{\sim} P(Y | X) P(X)$
        \STATE initialize $\widehat{\G} = \emptyset$
        \FOR{each pair of nodes $X_z, Y_j$}
            \STATE compute the independence test $\ind (X_z, Y_j)$
            \STATE if dependent add $(X_z, Y_j)$ to  $\widehat{\G}$
        \ENDFOR
        \STATE \textbf{Output}: causal dependency graph $\widehat{\G}$
    \end{algorithmic}
\end{algorithm}

Instead of the exact dependency graph $\G$, which can include dependencies that are too weak to be detected with a finite number of samples, we only address the dependencies above a given threshold.
\begin{definition}
\label{def:dependency_subgraph}
    We call $\G_\epsilon \subseteq \G$ the {\em $\epsilon$-dependency subgraph} of $\G$ if it holds, for each pair $(A, B) \in \G$ distributed as $P_{A, B}$, $ (A, B) \in \G_\epsilon $ iff $ \inf_{Q \in \{ \Delta_A \times \Delta_B\} } \| P_{A, B} - Q \|_1 \geq \epsilon.$
\end{definition}
Before presenting the result, we state the existence of a principled independence testing procedure.
\begin{lemma}[\citet{diakonikolas2021optimal}]
\label{thr:independence_testing}
    There exists an $(\epsilon,\delta)$-independence tester $\ind (A, B)$ for distributions $P_{A, B}$ on $[n] \times [n]$, which returns \emph{yes} if $A, B$ are independent, \emph{no} if $\inf_{Q \in \{ \Delta_A \times \Delta_B\} } \| P_{A, B} - Q \|_1 \geq \epsilon$, both with probability at least $1 - \delta$ and sample complexity $O ( n \log (1 / \delta) / \epsilon^2 )$.
\end{lemma}
We can now provide an upper bound to the number of samples required by a simple estimation procedure to return an $(\epsilon, \delta)$-estimate $\widehat{\G}$ of the causal dependency graph $\G$.
\begin{restatable}[]{theorem}{sampleComplexityG}
\label{thr:sample_complexity_G}
    Let $\mdp$ a discrete MDP with causal structure $\G$, let $\delta \in (0, 1)$, and let $\epsilon > 0$. The Algorithm~\ref{alg:learning_g} returns a dependency graph $\widehat{\G}$ such that $ Pr ( \widehat{\G} \neq \G_\epsilon ) \leq \delta $ with a sample complexity
    $
        K = O \big( n \log ( d_S^2 d_A / \delta) / \epsilon^2 \big).
    $
\end{restatable}
\begin{restatable}[]{corollary}{}
\label{thr:sample_complexity_G_tabular}
    Let $\mdp$ a tabular MDP. The result of Theorem~\ref{thr:sample_complexity_G} reduces to
    $
        K = O \big( \log ( S^2 A / \delta) / \epsilon^2 \big).
    $
\end{restatable}

\subsection{Sample Complexity of Learning the Bayesian Network of a Discrete MDP}
\label{sec:complexity_bayesian_network}
We present as a standalone result an upper bound to the sample complexity of learning the parameters of a Bayesian network $P_\G$ with a fixed structure $\G$. Especially, we refine known results~\citep[\eg][]{dasgupta1997sample} by considering the specific structure $\G$ of an MDP. 
If the structure $\G$ is dense, the number of parameters of $P_\G$ grows exponentially, making the estimation problem mostly intractable. Thus, we consider a $Z$-sparse $\G$ (Assumption~\ref{ass:sparseness}), as in previous works~\cite{dasgupta1997sample}.
Then, we can provide a polynomial sample complexity for the problem of learning the Bayesian network $P_\G$ of a an MDP $\mdp$.
\begin{restatable}[]{theorem}{sampleComplexityP}
\label{thr:sample_complexity_P}
    Let $\mdp$ a discrete MDP with causal structure $\G$, let $\delta \in (0, 1)$, and let $\epsilon > 0$. The Algorithm~\ref{alg:learning_p} returns a Bayesian network $\widehat{P}_\G$ such that $ Pr ( \| \widehat{P}_\G - P_\G \|_1 \geq \epsilon ) \leq \delta $ with a sample complexity
    $
        K = O \big( d_S^3 n^{3Z + 1} \log (  d_S n^Z / \delta ) / \epsilon^2  \big).
    $
\end{restatable}
\begin{restatable}[]{corollary}{sampleComplexityPtabular}
\label{thr:sample_complexity_P_tabular}
    Let $\mdp$ a tabular MDP. The result of Theorem~\ref{thr:sample_complexity_P} reduces to
    $
        K = O \big( S^2  2^{2Z} \log ( S 2^Z / \delta) / \epsilon^2 \big).
    $
\end{restatable}
\begin{algorithm}[t!]
    \caption{MDP Bayesian Network Estimation}
    \label{alg:learning_p}
    \begin{algorithmic}[t]
        \STATE \textbf{Input}: sampling model $P(Y|X)$, dependency graph $\G$, batch parameter $K$
        \STATE let $K' = \lceil K / d_S n^Z \rceil$
        \FOR{$j = 1, \ldots, d_S$}
            \STATE let $Z_j$ the scopes $(X[Z_j],Y[j]) \subseteq \G $
            \STATE initialize the counts $N (X[Z_j], Y[j]) = 0$
            \FOR{each value $x \in [n]^{|Z_j|}$}
                \FOR{$k = 1, \ldots, K'$}
                 \STATE draw $y \sim P(Y[j] | X[Z_j] = x)$
                 \STATE increment $N (X[Z_j] = x, Y[j] = y)$
                \ENDFOR
            \ENDFOR
            \STATE compute $\widehat{P}_j (Y[j]| X[Z_j]) = \frac{N (X[Z_j], Y[j])}{ K'}$
        \ENDFOR
        \STATE let $\widehat{P}_\G (Y | X) = \prod_{j = 1}^{d_S} \widehat{P}_j (Y[j] | X[Z_j])$
        \STATE \textbf{Output}: Bayesian network $\widehat{P}_\G$
    \end{algorithmic}
\end{algorithm}

\section{Numerical Validation}
\label{sec:experiments}
We empirically validate the theoretical findings of this work by experimenting on a synthetic example where each environment is a person, and the MDP represents how a series of actions the person can take influences their weight ($W$) and academic performance ($A$). As actions we consider hours of physical training ($P$), hours of sleep ($S$), hours of study ($St$), amount of vegetables in the diet ($D$), and the amount of caffeine intake ($C$). 
The obvious use-case for such a model would be a tracking device that monitors how the actions of a person influence their weight and academic performance and provides personalized recommendations to reach the person's goals. 
While the physiological responses of different individuals can vary, there are some underlying mechanisms shared by all humans, and therefore deemed causal in our terminology. Examples of such causal links are the dependency of weight on the type of diet, and the dependency of academic performance on the number of hours of study. 
Other links, such as the dependency of weight on the amount of caffeine, are present in some individuals, but are generally not shared and therefore not causal.
For simplicity, all variables are treated as discrete with values $0$ (below average), $1$ (average) or $2$ (above average). See Appendix B for details on how transition models of different environments are generated. A class $\class$ of 3 environments is used to estimate the causal transition model. All experiments are repeated 10 times and report the average and standard deviation. 

\begin{figure*}
    \includegraphics[]{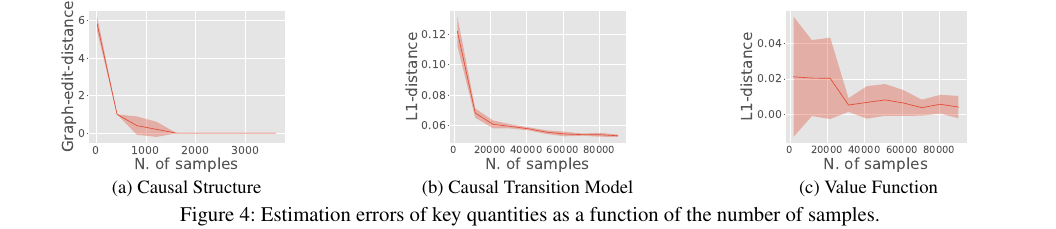}
\end{figure*}

\textbf{Causal Structure Estimation}~~~
We first empirically investigate the graph edit distance between estimated and ground-truth causal structures $\mathrm{GED}(\G, \widehat{\G})$ as a function of number of samples ($K'$ in Algorithm \ref{alg:learning_ctm}). The causal structure is estimated by obtaining the causal graph for each training environment (using a series of independence tests), and taking the intersection of their edges. As expected, the distance converges to zero as we increase the number of samples, and we can recover the exact causal graph (Figure 4a).

\textbf{Causal Transition Model Estimation}~~~
Figure 4b shows the $L_1$-distance between the estimated and ground-truth causal transition model, as a function of the number of samples ($K' + K''$ in Algorithm \ref{alg:learning_ctm}). As the samples grow, the $L_1$-distance shrinks towards $0.05$, which is due to the environments not fully respecting the evenness assumption.

\textbf{Value Function Estimation}~~~
Finally, we investigate whether we can approximate the optimal value function for an unseen environment.  
From Figure 4c, we observe that our algorithm is able to approximate the optimal value function up to a small error with a reasonable number of samples.

\section{Related Work}
\label{sec:related_work}

Finally, we revise the relevant literature and discuss how it relates with our problem formulation and results.

\textbf{Causal Discovery and Bayesian Networks}~~~
On a technical level, our work is related to previous efforts on the sample complexity of causal discovery~\cite{wadhwa2021sample} and Bayesian network estimation~\cite{friedman1996sample, dasgupta1997sample, bhattacharyya2022independence}.
None of these works consider the MDP setting. Instead, we account for the peculiar MDP structure to get sharper rates \wrt a blind application of previous results.

\textbf{Reward-Free RL}~~~
Reward-free RL~\cite{jin2020reward} is akin to a special case of our systematic generalization framework in which the set of MDPs is a singleton~\cite{wang2020reward, zanette2020provably, kaufmann2021adaptive, menard2021fast, zhang2021near, qiu2021reward}.
It is worth comparing our sample complexity result to independent reward-free exploration for each MDP. Let $|\uni| = U$, the latter would require at least $\Omega (U H^3 S^2 A / \epsilon^2)$ samples to obtain systematic generalization up to an $\epsilon$ threshold over a set of tabular MDPs $\uni$~\cite{jin2020reward}. This compares favorably with our rate $\widetilde{O} (M H^6 S^4 A^2 / \epsilon^2)$ whenever $U$ is small, but leveraging the inner structure of $\uni$ becomes crucial as $U$ grows to infinity, while $M$ remains constant. Our approach pays this further generality with the additional error term $\epsilon_\lambda$, which is unavoidable. It is an interesting direction to see whether additional factors in $S, A, H$  are also unavoidable.

\textbf{Hidden Structures in RL}~~~
Previous works have considered learning an hidden structure of the MDP for sample efficient RL~\cite{du2019provably, misra2020kinematic, misra2020provable, agarwal2020flambe}. Their focus is on learning latent representations of states assuming a linear structure in the MDP. This is orthogonal to our work, which instead targets the causal structure shared by infinitely many MDPs, while assuming access to the state features.
Other works~\cite[\eg][]{jin2020provably, cai2020provably, yin2022efficient} study the impact of structural properties of the MDP assuming access to the features. Our structural assumption is strictly more general than the linear structures they consider, but their work could provide useful inspiration to extend our results beyond discrete settings.

\textbf{Model-Based RL}~~~
Model-based RL~\cite{sutton2018reinforcement} prescribes learning an approximate model of the transition dynamics to extract an optimal policy. Theoretical works \citep[\eg][]{jaksch2010near, ayoub2020model} generally focus on the estimation of the approximate value functions obtained through the learned model, rather than the estimation of the model itself. A notable exception is \cite{tarbouriech2020active}, which targets point-wise high probability guarantees on the model estimation as we do in Lemma~\ref{thr:sample_complexity_ctm},~\ref{thr:sample_complexity_ctm_tabular}. However, they address the model estimation of a single MDP, instead of the shared transition dynamics of an infinite set of MDPs that we target in this paper. 

\textbf{Factored MDPs}~~~
The factored MDP formalism~\cite{kearns1999efficient} allows encoding transition dynamics that are the product of multiple independent factors. This is closely related to how we define the causal transition model in~\eqref{eq:factorization}, which can be seen as a factored MDP. Previous works have considered learning in factored MDPs, either assuming full knowledge of the factorization~\cite{delgado2011efficient, xu2020reinforcement, talebi2020improved, tian2020towards}, or by estimating its structure from data~\cite{strehl2007efficient, vigorito2009incremental, chakraborty2011structure, osband2014near, rosenberg2021oracle}. To the best of our knowledge, none of the existing works have considered the factored MDP framework in combination with a reward-free setting and systematic generalization, which bring unique challenges to the identification of the underlying factorization and the estimation of the transition factors.

\textbf{Causal RL}~~~
Previous works~\citep{zhang2020invariant,tomar2021model, gasse2021causal, feng2022factored} address model-based RL from a causal perspective. The motivations behind~\cite{zhang2020invariant} are especially similar to ours, but they have come to different structural assumptions, which lead to non-overlapping results.
To the best of our knowledge, we are the first to prove a polynomial sample complexity for causal model-based RL in systematic generalization. Similarly to our paper, \citet{feng2022factored} employ causal structure learning to build a factored representation of the MDP, but they tackle non-stationary changes in the environment instead of systematic generalization.
Finally,~\citet{lu2021efficient} show how to exploit a known causal representation for sample efficient RL, which can complement our work on how to learn such representation.

%

\nocite{*}
\bibliography{biblio}
\bibliographystyle{aaai23}

\clearpage
\onecolumn
\appendix

\section{Proofs}
\label{appendix:proofs}
\input{appendix_proofs}

\section{Numerical Validation}
\label{appendix:numerical validation}
\input{appendix_numerical_validation}

\end{document}

%% file: appendix_proofs.tex
\subsection*{Proofs of Section~\ref{sec:causal_perspective}}
\observationalDataSufficiency*
\begin{proof}
    First, recall that with observational data alone, a causal graph can be identified up to its Markov equivalence class~\cite{2014buhlmann}. This means that its skeleton and v-structure are properly identified, meanwhile determining the edge orientations requires interventional data in the general case. Since in the considered causal graph $\mathcal{G}$ the edges orientations are determined a priori (as they follow the direction of time), the causal graph can be entirely determined by using only observational data.
\end{proof}

\subsection*{Proofs of Section~\ref{sec:complexity_generalization}: Causal Transition Model Estimation}

Before reporting the proof of the main result in Theorem~\ref{thr:sample_complexity}, it is worth considering a set of lemmas that will be instrumental to the main proof.

First, we provide an upper bound to the L1-distance between the Bayesian network $P_\G$ over a given structure $\G$ and the Bayesian network $P_{\G_\epsilon}$ over the structure $\G_\epsilon$, which is the $\epsilon$-dependency subgraph of $\G$ as defined in Definition~\ref{def:dependency_subgraph}.

\begin{lemma}
\label{thr:l1_norm_bayesian_networks}
    Let $\G$ a $Z$-sparse dependency graph, and let $\G_\epsilon$ its corresponding $\epsilon$-dependence subgraph for a threshold $\epsilon > 0$. The L1-distance between the Bayesian network $P_{\G}$ over $\G$ and the Bayesian network $P_{\G_\epsilon}$ over $\G_\epsilon$ can be upper bounded as
    \begin{equation*}
        \| P_{\G} - P_{\G_\epsilon} \|_1 \leq d_S Z \epsilon.
    \end{equation*}
\end{lemma}
\begin{proof}
    The proof is based on the fact that every edge $(X_i, Y_j)$ such that $(X_i, Y_j) \in \G$ and $(X_i, Y_j) \notin \G_{\epsilon}$ corresponds to a weak conditional dependence (see Definition~\ref{def:dependency_subgraph}), which means that $ \| P_{Y_j | X_i} - P_{Y_j} \|_1 \leq \epsilon$. 
    
    We denote with $Z_j$ the scopes of the parents of the node $Y[j]$ in $\G$, \ie $\Pa_\G (Y[j]) = X[Z_j]$, and with $Z_{j, \epsilon}$ the scopes of the parents of the node $Y[j]$ in $\G_{\epsilon}$, \ie $\Pa_{\G_\epsilon} (Y[j]) = X[Z_{j,\epsilon}]$. As a direct consequence of Definition~\ref{def:dependency_subgraph}, we have $Z_{j, \epsilon} \subseteq Z_j$ for any $j \in d_S$, and we can write
    \begin{equation*}
    \begin{aligned}
        P_\G (Y | X) 
        = \prod_{j = 1}^{d_S} P_j (Y[j] \ | \ X[Z_j])
        = \prod_{j = 1}^{d_S} P_j (Y[j] \ | \ X[Z_{j,\epsilon}], X[Z_j \setminus Z_{j,\epsilon}]),
        \\
        P_{\G_{\epsilon}} (Y | X) 
        = \prod_{j = 1}^{d_S} P_j (Y[j] \ | \ X[Z_{j,\epsilon}]).
    \end{aligned}
    \end{equation*}
    Then, we let $Z_j \setminus Z_{j, \epsilon} = [I]$ overwriting the actual indices for the sake of clarity, and we derive
    \begin{align}
        \| P_\G - P_{\G_\epsilon} \|_1
        &\leq \sum_{j = 1}^{d_S} \Big\| P_j (Y[j] \ | \ X[Z_{j, \epsilon}], \cup_{i = 1}^I X[i]) - P_j (Y[j] \ | \ X[Z_{j, \epsilon}] ) \Big\|_1 \label{eq:13} \\ 
        &\leq \sum_{j = 1}^{d_S} \sum_{i' = 1}^{I} \Big\| P_j (Y[j] \ | \ X[Z_{j,\epsilon}], \cup_{i = i'}^I X[i]) - P_j (Y[j] \ | \ X[Z_{j, \epsilon}], \cup_{i = i' + 1}^I X[i] ) \Big\|_1 \label{eq:14} \\
        &\leq \sum_{j = 1}^{d_S} \sum_{i' = 1}^{I} \epsilon \leq d_S Z \epsilon, \label{eq:15}
    \end{align}
    in which we employed the property $\| \mu - \nu \|_1 \leq \| \prod_i \mu_i - \prod_i \nu_i \|_1 \leq \sum_{i} \| \mu_i - \nu_i \|_1$ for the L1-distance between product distributions $\mu = \prod_i \mu_i, \nu = \prod_i \nu_i$ to write~\eqref{eq:13}, we repeatedly applied the triangle inequality $\| \mu - \nu \|_1 \leq \| \mu - \rho \|_1 + \| \rho - \nu \|_1$ to get~\eqref{eq:14} from~\eqref{eq:13}, we upper bounded each term of the sum in~\eqref{eq:14} with $\epsilon$ thanks to Definition~\ref{def:dependency_subgraph}, and we finally employed the $Z$-sparseness Assumption~\ref{ass:sparseness} to upper bound $I$ with $Z$ in~\eqref{eq:15}.
\end{proof}

Next, we provide a crucial sample complexity result for a provably efficient estimation of a Bayesian network $\widehat{P}_{\widehat{\G}}$ over an estimated $\epsilon$-dependency subgraph $\widehat{\G}$, which relies on both the causal structure estimation result of Theorem~\ref{thr:sample_complexity_G} and the Bayesian network estimation result of Theorem~\ref{thr:sample_complexity_P}.

\begin{lemma}
\label{thr:sample_complexity_PG}
    Let $\mdp$ be a discrete MDP, let $\class = \{ \mdp \}$ be a singleton class, let $\delta \in (0, 1)$, and let $\epsilon > 0$. The Algorithm~\ref{alg:learning_ctm} returns a Bayesian network $\widehat{P}_{\widehat{\G}}$ such that $ Pr ( \| \widehat{P}_{\widehat{\G}} - P_\G \|_1 \geq \epsilon ) \leq \delta $ with a sample complexity
    \begin{equation*}
        K = O \Bigg( \frac{d_S^3  Z^2  n^{3Z + 1}  \log \big( \frac{ 4 d_S^2 d_A n^Z }{ \delta} \big) }{ \epsilon^2} \Bigg).
    \end{equation*}
\end{lemma}
\begin{proof}
    We aim to obtain the number of samples $K = K' + K''$ for which Algorithm~\ref{alg:learning_ctm} is guaranteed to return a Bayesian network estimate $\widehat{P}_{\widehat{\G}}$ over a causal structure estimate $\widehat{\G}$ such that $Pr ( \| \widehat{P}_{\widehat{\G}} - P_\G \|_1 \geq \epsilon ) \leq \delta$ in a setting with a singleton class of discrete MDPs. First, we derive the following decomposition of the error
    \begin{equation}
        \| \widehat{P}_{\widehat{\G}} - P_\G \|_1
        \leq \| \widehat{P}_{\widehat{\G}} \pm P_{\widehat{\G}} \pm P_{{\G}_{\epsilon'}} - P_\G \|_1
        \leq \| \widehat{P}_{\widehat{\G}} - P_{\widehat{\G}} \|_1
        + \| P_{\widehat{\G}} - P_{\G_{\epsilon'}} \|_1
        + \| P_{\G_{\epsilon'}} - P_\G \|_1
        \label{eq:error_decomposition}
    \end{equation}
    in which we employed the triangle inequality $\| \mu - \nu \|_1 \leq \| \mu - \rho \|_1 + \| \rho - \nu \|_1$. Then, we can write
    \begin{equation*}
        Pr \big( \| \widehat{P}_{\widehat{\G}} - P_\G \|_1 \geq \epsilon \big)
        \leq \underbrace{Pr \Big( \| \widehat{P}_{\widehat{\G}} - P_{\widehat{\G}} \|_1 \geq \frac{\epsilon}{3} \Big)}_{\text{Bayesian network estimation } (\star)}
        + \underbrace{Pr \Big( \| P_{\widehat{\G}} - P_{\G_{\epsilon'}} \|_1 \geq \frac{\epsilon}{3}\Big)}_{\text{causal structure estimation } (\bullet)}
        + \underbrace{Pr \Big( \| P_{\G_{\epsilon'}} - P_\G \|_1 \geq \frac{\epsilon}{3} \Big)}_{\text{Bayesian network subgraph } (\diamond)}
    \end{equation*}
    through the decomposition~\eqref{eq:error_decomposition} and a union bound to isolate the three independent sources of error $(\star), (\bullet), (\diamond)$. To upper bound the latter term $(\diamond)$ with 0, we invoke Lemma~\ref{thr:l1_norm_bayesian_networks} to have $d_s Z \epsilon' \leq \frac{\epsilon}{3}$, which gives $\epsilon' \leq \frac{\epsilon}{3 d_S Z}$. Then, we consider the middle term $(\bullet)$, for which we can write
    \begin{equation}
        Pr \bigg( \| P_{\widehat{\G}} - P_{\G_{\epsilon'}} \|_1 \geq \frac{\epsilon}{3} \bigg) \leq Pr \big( \widehat{\G} \neq \G_{\epsilon'} \big).
        \label{eq:causal_estimation}
    \end{equation}
    We can now upper bound $(\bullet) \leq \delta / 2$ through~\eqref{eq:causal_estimation} by invoking Theorem~\ref{thr:sample_complexity_G} with threshold $\epsilon' = \frac{\epsilon}{3 d_S Z}$ and confidence $\delta' = \frac{\delta}{2}$, which gives
    \begin{equation}
        K' = C' \bigg( \frac{ d_S^{4 / 3} Z^{4 / 3} n \log^{1 / 3} (2 d_S^2 d_A / \delta) }{ \epsilon^{4 / 3}} + \frac{d_S^2 Z^2 n \log^{1 / 2} (2 d_S^2 d_A / \delta) + \log(2 d_S^2 d_A / \delta)}{\epsilon^2} \bigg).
        \label{eq:k1}
    \end{equation}
    Next, we can upper bound $(\star) \leq \delta / 2$ by invoking Theorem~\ref{thr:sample_complexity_P} with threshold $\epsilon' = \frac{\epsilon}{3}$ and confidence $\delta' = \frac{\delta}{2}$, which gives
    \begin{equation}
        K'' = C'' \bigg( \frac{  d_S^3  n^{3Z + 1}  \log (4 d_S n^Z / \delta)}{\epsilon^2 } \bigg).
        \label{eq:k2}
    \end{equation}
    Finally, through the combination of~\eqref{eq:k1} and~\eqref{eq:k2}, we can derive the sample complexity that guarantees $ Pr ( \| \widehat{P}_{\widehat{\G}} - P_\G \|_1 \geq \epsilon ) \leq \delta $ under the assumption $\epsilon^{4 / 3} \ll \epsilon^2$, \ie
    \begin{equation*}
        K = K' + K'' \leq \frac{d_S^3 Z^2 n^{3Z + 1} \log \Big( \frac{4 d_S^2 d_A n^Z}{\delta} \Big)}{\epsilon^2},
    \end{equation*}
    which concludes the proof.
\end{proof}

Whereas Lemma~\ref{thr:sample_complexity_PG} is concerned with the sample complexity of learning the Bayesian network of a singleton class, we can now extend the result to account for a class $\class$ composed of $M$ discrete MDPs.

\sampleComplexityCTM*
\begin{proof}
    We aim to obtain the number of samples $K = M K' + K''$ for which Algorithm~\ref{alg:learning_ctm} is guaranteed to return a Bayesian network estimate $\widehat{P}_{\widehat{\G}}$ over a causal structure estimate $\widehat{\G}$ such that $Pr ( \| \widehat{P}_{\widehat{\G}} - P_\G \|_1 \geq \epsilon ) \leq \delta$ in a setting with a class of $M$ discrete MDPs. First, we can derive an analogous decomposition as in~\eqref{eq:error_decomposition}, such that we have
    \begin{equation*}
        Pr \big( \| \widehat{P}_{\widehat{\G}} - P_\G \|_1 \geq \epsilon \big)
        \leq \underbrace{Pr \Big( \| \widehat{P}_{\widehat{\G}} - P_{\widehat{\G}} \|_1 \geq \frac{\epsilon}{3} \Big)}_{\text{Bayesian network estimation } (\star)}
        + \underbrace{Pr \Big( \| P_{\widehat{\G}} - P_{\G_{\epsilon'}} \|_1 \geq \frac{\epsilon}{3}\Big)}_{\text{causal structure estimation } (\bullet)}
        + \underbrace{Pr \Big( \| P_{\G_{\epsilon'}} - P_\G \|_1 \geq \frac{\epsilon}{3} \Big)}_{\text{Bayesian network subgraph } (\diamond)}
    \end{equation*}
    through a union bound. Crucially, the terms $(\star), (\diamond)$ are unaffected by the class size, which leads to $K'' = \eqref{eq:k2}$ by upper bounding $(\star)$, and $\epsilon' \leq \frac{\epsilon}{3 d_S Z}$ by upper bounding $(\diamond)$, exactly as in the proof of Lemma~\ref{thr:sample_complexity_PG}. Instead, the number of samples $K'$ has to guarantee that $(\bullet) = Pr ( \| P_{\widehat{\G}} - P_{\G_{\epsilon'}} \|_1 \geq \epsilon / 3 ) \leq \delta / 2$, where the causal structure $\G_{\epsilon'}$ is now the intersection of the causal structures of the single class components $\mdp_i$, \ie $\G_{\epsilon'} = \cap_{i = 1}^M \G_{\epsilon', i}$. Especially, we can write
    \begin{equation}
        (\bullet) = Pr \Big( \| P_{\widehat{\G}} - P_{\G_{\epsilon'}} \|_1 \geq \frac{\epsilon}{3} \Big)
        \leq Pr \Big( \widehat{\G} \neq \G_{\epsilon'} \Big)
        \leq Pr \bigg( \bigcup_{i = 1}^M \widehat{\G}_i \neq \G_{\epsilon', i} \bigg)
        \leq \sum_{i = 0}^M Pr \Big( \widehat{\G}_i \neq \G_{\epsilon', i} \Big),
        \label{eq:causal_structures}
    \end{equation}
    through a union bound on the estimation of the single causal structures $\widehat{\G}_i$. Then, we can upper bound $(\bullet) \leq \delta / 2$ through~\eqref{eq:causal_structures} by invoking Theorem~\ref{thr:sample_complexity_G} with threshold $\epsilon' = \frac{\epsilon}{3 d_S Z}$ and confidence $\delta' = \frac{\delta}{2M}$, which gives
    \begin{equation}
        K' = C' \bigg( \frac{ d_S^{4 / 3}  Z^{4 / 3}  n  \log^{1 / 3} (2 M d_S^2 d_A / \delta) }{ \epsilon^{4 / 3}} + \frac{d_S^2  Z^2  n  \log^{1 / 2} (2 M d_S^2 d_A / \delta) + \log(2 M d_S^2 d_A / \delta)}{\epsilon^2} \bigg).
        \label{eq:k1m}
    \end{equation}
    Finally, through the combination of~\eqref{eq:k1m} and~\eqref{eq:k2}, we can derive the sample complexity that guarantees $ Pr ( \| \widehat{P}_{\widehat{\G}} - P_\G \|_1 \geq \epsilon ) \leq \delta $ under the assumption $\epsilon^{4 / 3} \ll \epsilon^2$, \ie
    \begin{equation*}
        K = M K' + K'' \leq \frac{M d_S^3  Z^2  n^{3Z + 1}  \log \Big( \frac{4 M d_S^2 d_A n^Z}{\delta} \Big)}{\epsilon^2},
    \end{equation*}
    which concludes the proof.
\end{proof}

It is now straightforward to extend Lemma~\ref{thr:sample_complexity_ctm} for a class $\class$ composed of $M$ tabular MDPs.

\sampleComplexityCTMtabular*
\begin{proof}
    To obtain $K = M K' + K''$, we follows similar steps as in the proof of Lemma~\ref{thr:sample_complexity_ctm}, to have the usual decomposition of the event $Pr ( \| \widehat{P}_{\widehat{\G}} - P_\G \|_1 \geq \epsilon ) $ in the $(\star), (\bullet), (\diamond)$ terms. We can deal with $(\diamond)$ as in Lemma~\ref{thr:sample_complexity_ctm} to get $\epsilon' \leq \frac{\epsilon}{3 S Z}$. Then, we upper bound $(\bullet) \leq \delta / 2$ by invoking Corollary~\ref{thr:sample_complexity_G_tabular} (instead of Theorem~\ref{thr:sample_complexity_G}) with threshold $\epsilon' = \frac{\epsilon}{3 S Z}$ and confidence $\delta' = \frac{\delta}{2M}$, which gives
    \begin{equation}
        K' = C' \bigg( \frac{ S^{4 / 3}  Z^{4 / 3} \log^{1 / 3} (2 M S^2 A / \delta) }{ \epsilon^{4 / 3}} + \frac{S^2  Z^2  \log^{1 / 2} (2 M S^2 A / \delta) + \log(2 M S^2 A / \delta)}{\epsilon^2} \bigg).
        \label{eq:k1m_tabular}
    \end{equation}
    Similarly, we upper bound $(\star) \leq \delta / 2$ by invoking Corollary~\ref{thr:sample_complexity_P_tabular} (instead of Theorem~\ref{thr:independence_testing}) with threshold $\epsilon' = \frac{\epsilon}{3}$ and confidence $\delta' = \frac{\delta}{2}$, which gives
    \begin{equation}
        K'' = \frac{ 18  S^2 2^{2Z}  \log( 4 S 2^Z / \delta)}{ \epsilon^2 }.
        \label{eq:k2_tabular}
    \end{equation}
    Finally, we combine~\ref{eq:k1m_tabular} with~\ref{eq:k2_tabular} to obtain
    \begin{equation*}
        K = M K' + K'' \leq \frac{M S^2 Z^2 2^{2Z} \log \big( \frac{ 4 M S^2 A 2^Z }{ \delta} \big) }{ \epsilon^2} 
    \end{equation*}
\end{proof}

\subsection*{Proofs of Section~\ref{sec:complexity_generalization}: Planning}

\sampleComplexity*
\begin{proof}
    Consider the MDPs with transition model $P$ and $\widehat{P}_{\widehat{\G}}$. We refer to the respective optimal policies as $\pi^*$ and $\widehat{\pi}^*$. Moreover, since the reward $r$ is fixed, we remove it from the expressions for the sake of clarity, and refer with $\widehat{V}$ to the value function of the MDP with transition model $\widehat{P}_{\widehat{\G}}$. As done in~\citep[][Theorem 3.5]{jin2020reward}, we can write the following decomposition with $V^* := V^{\pi^*}$,
    \begin{align*}
         \EV_{s_1 \sim P}\Big[V^*_1 (s_1) - V^{\widehat{\pi}}_1 (s_1)\Big] &\leq \underbrace{\Big|\EV_{s_1 \sim P}\Big[V^*_1 (s_1) - \widehat{V}^{\widehat{\pi}^*}_1 (s_1)\Big]\Big|}_{\text{evaluation error}} + \underbrace{\EV_{s_1 \sim P}\Big[\widehat{V}^*_1 (s_1) - \widehat{V}^{\widehat{\pi}^*}_1 (s_1)\Big]}_{\text{$\leq 0$ by def.}}\\
         &+ \underbrace{\EV_{s_1 \sim P}\Big[\widehat{V}^{\widehat{\pi}^*}_1 (s_1) - \widehat{V}^{\widehat{\pi}}_1 (s_1)\Big]}_{\text{optimization error}} + \underbrace{\Big| \EV_{s_1 \sim P}\Big[\widehat{V}^{\widehat{\pi}}_1 (s_1) - V^{\widehat{\pi}}_1(s_1)\Big]\Big|}_{\text{evaluation error}}\\
         &\leq \underbrace{2 n^{Z + 1} H^3\epsilon'}_{\epsilon} + \underbrace{2n^{2Z + 1}d_S H^3\lambda}_{\epsilon_\lambda}
    \end{align*}
    where in the last step we have set to 0 the approximation due to the planning oracle assumption, and we have bounded the evaluation errors according to Lemma \ref{thr:evaluation_error_discrete}. In order to get $2 n^{Z + 1}H^3\epsilon' = \epsilon$ we have to set $\epsilon' = \frac{\epsilon}{2n^{Z + 1}H^3}$. Considering the sample complexity result in Lemma \ref{thr:sample_complexity_ctm} the final sample complexity will be
    \begin{equation*}
        K = O \Bigg( \frac{M d_S^3 Z^2 n^{3Z + 1} \log \big( \frac{ 4 M d_S^2 d_A n^Z }{ \delta} \big) }{ (\epsilon')^2} \Bigg) = O \Bigg( \frac{4  M d_S^3 Z^2 n^{5Z + 3} H^6 \log \big( \frac{ 4 M d_S^2 d_A n^Z }{ \delta} \big) }{ \epsilon^2} \Bigg).
    \end{equation*}
\end{proof}

\begin{lemma}
    \label{thr:evaluation_error_discrete}
    Under the preconditions of Theorem~\ref{thr:sample_complexity}, with probability $1-\delta$, for any reward function $r$ and policy $\pi$, we can bound the value function estimation error as follows.
    \begin{equation}
        \Big| \EV_{s \sim P}\Big[\widehat{V}^\pi_{1,r}(s) - V_{1,r}^\pi(s) \Big] \Big| \leq \underbrace{n^{Z + 1} H^3\epsilon'}_{\epsilon} + \underbrace{n^{2Z + 1} d_S H^3\lambda}_{\epsilon_\lambda}
    \end{equation}
    where $\widehat{V}$ is the value function of the MDP with transition model $\widehat{P}_{\widehat{\G}}$, $\epsilon'$ is the approximation error between $\widehat{P}_{\widehat{\G}}$ and $P_\G$ studied in Lemma \ref{thr:sample_complexity_ctm}, and $\lambda$ stands for the $\lambda$-sufficiency parameter of $P_\G$.
\end{lemma}
\begin{proof}
    The proof will be along the lines of that of Lemma 3.6 in~\citep[]{jin2020reward}. We first recall ~\citep[][Lemma E.15]{dann2018unifying}, which we restate in Lemma \ref{thr:V_difference}. In this proof, we consider an environment specific true MDP $\mdp$ with transition model $P$, and an MDP $\widehat{\mdp}$ that has as transition model the estimated causal transition model $\widehat{P}_{\widehat{\G}}$. In the following, the expectations will be \wrt $P$. Moreover, since the reward $r$ is fixed, we remove it from the expressions for the sake of clarity. We can start deriving
    \begin{align}
        \Big| \EV_{s \sim P}\Big[\widehat{V}^\pi_1(s) - V_1^\pi(s) \Big] \Big| &\leq \Big| \EV_{X}\Big[\sum_{h=1}^H (\widehat{P}_{\widehat{\G}} - P)\widehat{V}_{h+1}^\pi(X) \Big] \Big| \nonumber\\
        &\leq  \EV_{X}\Big[\sum_{h=1}^H \Big|(\widehat{P}_{\widehat{\G}} - P)\widehat{V}_{h+1}^\pi(X)\Big| \Big] \nonumber\\
        &= \sum_{h=1}^H \EV_{X}\Big|(\widehat{P}_{\widehat{\G}} - P)\widehat{V}_{h+1}^\pi(X)\Big| \label{initial_expression}.
    \end{align}
    We now bound a single term within the sum above as follows
    \begin{align}
        \EV_{X}\Big|(\widehat{P}_{\widehat{\G}} - P)\widehat{V}_{h+1}^\pi(X)\Big| &= \EV_{X}\Big|(\widehat{P}_{\widehat{\G}} - P_\G + P_\G - P)\widehat{V}_{h+1}^\pi(X)\Big| \nonumber\\
        &= \EV_{X}\Big|(\widehat{P}_{\widehat{\G}} - P_\G)\widehat{V}_{h+1}^\pi(X) + (P_\G - P)\widehat{V}_{h+1}^\pi(X)\Big| \nonumber\\
        &\leq \EV_{X} \Bigg[\Big|(\widehat{P}_{\widehat{\G}} - P_\G)\widehat{V}_{h+1}^\pi(X)\Big| + \Big|(P_\G - P)\widehat{V}_{h+1}^\pi(X)\Big| \Bigg] \nonumber\\
        &=\EV_{X} \Big|(\widehat{P}_{\widehat{\G}} - P_\G)\widehat{V}_{h+1}^\pi(X)\Big| + \EV_{X}\Big|(P_\G - P)\widehat{V}_{h+1}^\pi(X)\Big| \label{eq_decomposition}.
    \end{align}
    We can now bound each term. Let us start considering the first one 
    \begin{align}
        \EV_{X}\Big|(\widehat{P}_{\widehat{\G}} - P_\G)\widehat{V}_{h+1}^\pi(X)\Big| 
        &= \EV_{X}\Big|\widehat{P}_{\widehat{\G}}\widehat{V}_{h+1}^\pi(X) - P_\G\widehat{V}_{h+1}^\pi(X)\Big| \nonumber\\
        &= \EV_{X}\Big|\sum_{Y}\widehat{P}_{\widehat{\G}}(Y|X)\widehat{V}_{h+1}^\pi(Y) - \sum_{Y}P_\G(Y|X) \widehat{V}_{h+1}^\pi(Y)\Big| \nonumber\\
        &= \EV_{X}\Big|\sum_{Y}\widehat{P}_{\widehat{\G}}(Y|X)\EV_{X' \sim \pi}\Big[ r(X') + \widehat{P}_{\widehat{\G}}\widehat{V}_{h+2}^\pi(X') \Big] \nonumber \\
        &\quad- \sum_{Y}P_\G(Y|X)\EV_{X' \sim \pi}\Big[ r(X') + P_\G \widehat{V}_{h+2}^\pi(X') \Big] \Big| \nonumber\\
        &= \EV_{X}\Big|\sum_{Y}\big(\widehat{P}_{\widehat{\G}}(Y|X) - P_\G(Y|X)\big)\EV_{X' \sim \pi}\Big[ r(X')\Big] \nonumber \\
        &\quad+ \sum_{Y}\widehat{P}_{\widehat{\G}}(Y|X)\EV_{X' \sim \pi}\Big[ \widehat{P}_{\widehat{\G}} \widehat{V}_{h+2}^\pi(X')\Big]
        - \sum_{Y}P_\G(Y|X)\EV_{X' \sim \pi}\Big[ P_\G \widehat{V}_{h+2}^\pi(X')\Big]  \Big| \nonumber \\
        &\leq \EV_{X}\Big|\sum_{Y}\big(\widehat{P}_{\widehat{\G}}(Y|X) - P_\G(Y|X)\big) \Big| \label{eq_3_terms} \\
        &\quad+ \EV_{X} \Big|\sum_{Y}\widehat{P}_{\widehat{\G}}(Y|X)\EV_{X' \sim \pi}\Big[ \widehat{P}_{\widehat{\G}} \widehat{V}_{h+2}^\pi(X')\Big] 
        - \sum_{Y}P_\G(Y|X)\EV_{X' \sim \pi}\Big[ P_\G \widehat{V}_{h+2}^\pi(X')\Big]  \Big| \nonumber.
    \end{align}
    We can now bound the first term of \eqref{eq_3_terms}
    \begin{align}
        \EV_{X}\Big|\sum_{Y}\big(\widehat{P}_{\widehat{\G}}(Y|X) - P_\G(Y|X)\big) \Big|
        &= \EV_{X}\Bigg|\sum_{Y}\Big(\prod_{j=1}^{d_S}\widehat{P}_j(Y[j]|X[Z_j]) - \prod_{j=1}^{d_S}P_j(Y[j]|X[Z_j])\Big) \Bigg| \nonumber\\
        &\leq \EV_{X}\Bigg[\sum_{Y}\sum_{j=1}^{d_S}\Big| \widehat{P}_j(Y[j]|X[Z_j]) - P_j(Y[j]|X[Z_j]) \Big| \Bigg] \nonumber\\
        &= \sum_{X}P_\G^\pi(X)\Bigg[\sum_{Y}\sum_{j=1}^{d_S}\Big| \widehat{P}_j(Y[j]|X[Z_j]) - P_j(Y[j]|X[Z_j]) \Big| \Bigg] \nonumber\\
        &= \sum_{Y}\sum_{j=1}^{d_S}\sum_{X[Z_j]}P_\G^\pi(X[Z_j])\Big| \widehat{P}_j(Y[j]|X[Z_j]) - P_j(Y[j]|X[Z_j]) \Big| \label{single_term}.
    \end{align}
    Due to the uniform sampling and Z-sparseness assumptions, we have $P_\G(X[Z_j]) = \frac{1}{n^Z}$, hence
    \begin{equation*}
        \max_{\pi^\dagger}\frac{P_\G^{\pi^\dagger}(X[Z_j])}{P_\G(X[Z_j])} \leq \frac{1}{P_\G(X[Z_j])} = n^Z.
    \end{equation*}
    Therefore, we have
    $
        P_\G^{\pi^\dagger}(X[Z_j]) \leq n^Z\cdot P_\G(X[Z_j]).
    $
    Replacing this in  \eqref{single_term} and marginalizing over $Y$ \textbackslash $Y[j]$ we obtain
    \begin{align*}
        \EV_{X}\Big|\sum_{Y}\big(\widehat{P}_{\widehat{\G}}(Y|X) - P_\G(Y|X)\big) \Big| 
        &= n^Z \sum_{j=1}^{d_s}\sum_{Y[j]}\sum_{X[Z_j]}\Big| \widehat{P}_j(Y[j]|X[Z_j]) - P_j(Y[j]|X[Z_j]) \Big| P_\G(X[Z_j])\\
        &\leq n^Z \sum_{j=1}^{d_S}\sum_{Y[j]} \frac{\epsilon'}{d_S}\sum_{X[Z_j]}P_\G(X[Z_j])\\
        &= n^{Z+1} \epsilon',
    \end{align*}
    where $\frac{\epsilon'}{d_S}$ is the approximation term of each component. By plugging this bound into \eqref{eq_3_terms} we get
    \begin{align*}
        \EV_{X}\Big|(\widehat{P}_{\widehat{\G}} - P_\G)\widehat{V}_{h+1}^\pi(X)\Big|
        &\leq n^{Z+1} \epsilon' + \EV_{X} \Big|\sum_{Y}\widehat{P}_{\widehat{\G}}(Y|X)\EV_{X' \sim \pi}\Big[ \widehat{P}_{\widehat{\G}} \widehat{V}_{h+2}^\pi(X')\Big] 
        - \sum_{Y}P_\G(Y|X)\EV_{X' \sim \pi}\Big[ P_\G \widehat{V}_{h+2}^\pi(X')\Big]  \Big|\\
        &\leq \sum_{i=h+1}^H i\cdot n^{Z+1} \epsilon'\\
        &\leq H^2 n^{Z+1} \epsilon'
    \end{align*}
    where in the last step we have recursively bounded the right terms as in \eqref{recursive_bound}.
    By considering $2Z$-sparseness, $\lambda$-sufficiency, and that the transition model $P$ factorizes, we can apply the same procedure to bound the second term of equation \eqref{eq_decomposition} as
    \begin{equation*}
        \EV_{X}\Big|(P_\G - P)\widehat{V}_{h+1}^\pi(X)\Big| \leq H^2 n^{Z + 1} d_S \lambda.
    \end{equation*}
    Therefore, the initial expression in \eqref{initial_expression} becomes
    \begin{align}
        \Big| \EV_{s \sim P}\Big[\widehat{V}^\pi_1(s) - V_1^\pi(s) \Big] \Big| &\leq \sum_{h=1}^H \EV_{X}\Big|(\widehat{P}_{\widehat{\G}} - P)\widehat{V}_{h+1}^\pi(X)\Big|\\
        &\leq \sum_{h=1}^H [n^{Z + 1} H^2 \epsilon' + n^{2Z + 1} d_S H^2 \lambda]\\
        &\leq \underbrace{n^{Z + 1} H^3 \epsilon'}_{\epsilon} + \underbrace{n^{2Z + 1}d_S H^3\lambda}_{\epsilon_\lambda}.
    \end{align}
\end{proof}

\sampleComplexitytabular*
\begin{proof}
    Consider the MDPs with transition model $P$ and $\widehat{P}_{\widehat{\G}}$. We refer to the respective optimal policies as $\pi^*$ and $\widehat{\pi}^*$. Moreover, since the reward $r$ is fixed, we remove it from the expressions for the sake of clarity, and refer with $\widehat{V}$ to the value function of the MDP with transition model $\widehat{P}_{\widehat{\G}}$. As done in~\citep[][Theorem 3.5]{jin2020reward}, we can write the following decomposition with $V^* := V^{\pi^*}$
    \begin{align*}
         \EV_{s_1 \sim P}\Big[V^*_1 (s_1) - V^{\widehat{\pi}}_1 (s_1)\Big] &\leq \underbrace{\Big|\EV_{s_1 \sim P}\Big[V^*_1 (s_1) - \widehat{V}^{\widehat{\pi}^*}_1 (s_1)\Big]\Big|}_{\text{evaluation error}} + \underbrace{\EV_{s_1 \sim P}\Big[\widehat{V}^*_1 (s_1) - \widehat{V}^{\widehat{\pi}^*}_1 (s_1)\Big]}_{\text{$\leq 0$ by def.}}\\
         &+ \underbrace{\EV_{s_1 \sim P}\Big[\widehat{V}^{\widehat{\pi}^*}_1 (s_1) - \widehat{V}^{\widehat{\pi}}_1 (s_1)\Big]}_{\text{optimization error}} + \underbrace{\Big| \EV_{s_1 \sim P}\Big[\widehat{V}^{\widehat{\pi}}_1 (s_1) - V^{\widehat{\pi}}_1}_{\text{evaluation error}} (s_1)\Big]\Big|\\
         &\leq \underbrace{2SAH^3\epsilon'}_{\epsilon} + \underbrace{2SAH^3\lambda}_{\epsilon_\lambda}
    \end{align*}
    where in the last step we have set to 0 the approximation due to the planning oracle assumption, and we have bounded the evaluation errors according to Lemma \ref{thr:evaluation_error_tabular}. In order to get $2SAH^3 \epsilon' = \epsilon$ we have to set $\epsilon' = \frac{\epsilon}{2SAH^3}$. Considering the sample complexity result in Lemma \ref{thr:sample_complexity_ctm_tabular} the final sample complexity will be
    \begin{equation*}
        K = O \Bigg( \frac{M \ S^2 \ Z^2 \ 2^{2Z} \ \log \big( \frac{ 4 M S^2 A 2^Z }{ \delta} \big) }{ (\epsilon')^2} \Bigg) = O \Bigg( \frac{4M \ S^4 \ A^2 \ H^6 \ Z^2 \ 2^{2Z} \ \log \big( \frac{ 4 M S^2 A 2^Z }{ \delta} \big) }{ \epsilon^2} \Bigg).
    \end{equation*}
\end{proof}

\begin{lemma}
    \label{thr:evaluation_error_tabular}
    Under the preconditions of Corollary~\ref{thr:sample_complexity_tabular}, with probability $1-\delta$, for any reward function $r$ and policy $\pi$, we can bound the value function estimation error as follows.
    \begin{equation}
        \Big| \EV_{s \sim P}\Big[\hat{V}^\pi_{1,r}(s) - V_{1,r}^\pi(s) \Big] \Big| \leq \underbrace{SAH^3\epsilon'}_{\epsilon} + \underbrace{SAH^3\lambda}_{\epsilon_\lambda}
    \end{equation}
    where $\widehat{V}$ is the value function of the MDP with transition model $\widehat{P}_{\widehat{\G}}$, $\epsilon'$ is the approximation error between $\widehat{P}_{\widehat{\G}}$ and $P_G$ studied in Lemma \ref{thr:sample_complexity_ctm}, and $\lambda$ stands for the $\lambda$-sufficiency parameter of $P_\G$.
\end{lemma}
\begin{proof}
    The proof will be along the lines of that of Lemma 3.6 in~\citep[]{jin2020reward}. We first recall ~\citep[][Lemma E.15]{dann2018unifying}, which we restate in Lemma \ref{thr:V_difference}. In this proof, we consider an environment specific true MDP $\mdp$ with transition model $P$, and an MDP $\widehat{\mdp}$ that has as transition model the estimated causal transition model $\widehat{P}_{\widehat{\G}}$. In the following, the expectations will be \wrt $P$. Moreover, since the reward $r$ is fixed, we remove it from the expressions for the sake of clarity. We can start deriving
    \begin{align*}
        \Big| \EV_{s \sim P}\Big[\hat{V}^\pi_1(s) - V_1^\pi(s) \Big] \Big| &\leq \Big| \EV_{\pi}\Big[\sum_{h=1}^H (\widehat{P}_{\widehat{\G}} - P)\widehat{V}_{h+1}^\pi(s_h,a_h) \Big] \Big|\\
        &\leq  \EV_{\pi}\Big[\sum_{h=1}^H \Big|(\widehat{P}_{\widehat{\G}} - P)\widehat{V}_{h+1}^\pi(s_h,a_h)\Big| \Big]\\
        &= \sum_{h=1}^H \EV_{\pi}\Big|(\widehat{P}_{\widehat{\G}} - P)\widehat{V}_{h+1}^\pi(s_h,a_h)\Big|.
    \end{align*}
    We now bound a single term within the sum above as follows
    \begin{align*}
        \EV_{\pi}\Big|(\widehat{P}_{\widehat{\G}} - P)\widehat{V}_{h+1}^\pi(s_h,a_h)\Big| &\leq \sum_{s,a} \Big|(\widehat{P}_{\widehat{\G}} - P)\widehat{V}^\pi(s,a)\Big| P^\pi(s,a)\\
        &= \sum_{s,a} \Big|(\widehat{P}_{\widehat{\G}} - P)\widehat{V}^\pi(s,a)\Big| P^\pi(s)\pi(a|s)\\
        &\leq \max_{\pi'}\sum_{s,a} \Big|(\widehat{P}_{\widehat{\G}} - P)\widehat{V}^\pi(s,a)\Big| P^\pi(s)\pi'(a|s)\\
        &= \max_{\nu: \Sspace \to \Aspace} \sum_{s,a} \Big|(\widehat{P}_{\widehat{\G}} - P)\widehat{V}^\pi(s,a)\Big| P^\pi(s) \one \{a=\nu(s)\},
    \end{align*}
    where in the last step we have used the fact that there must exist an optimal deterministic policy.\newline
    Due to the uniform sampling assumption, we have $P(s,a) = \frac{1}{SA}$, hence
    \begin{equation*}
        \max_{\pi^\dagger}\frac{P^{\pi^\dagger}(s,a)}{P(s,a)} \leq \frac{1}{P(s,a)} = SA.
    \end{equation*}
    Therefore,
    $
        P^{\pi_\dagger}(s,a) \leq SA\cdot P(s,a).
    $
    Moreover, notice that, since $\pi'$ is deterministic we have $P^\pi(s) = P^{\pi'}(s) = P^{\pi'}(s,a) \leq SA\cdot P(s,a)$. Replacing it in the expression above we get
    \begin{align}
        \EV_{\pi}\Big|(\widehat{P}_{\widehat{\G}} - P)\widehat{V}_{h+1}^\pi(s_h,a_h)\Big| &\leq SA \cdot \sum_{s,a} \Big|(\widehat{P}_{\widehat{\G}} - P)\widehat{V}_{h+1}^\pi(s,a)\Big| P(s) \one \{a=\nu(s)\} \nonumber\\
        &\leq SA \cdot \Big|(\widehat{P}_{\widehat{\G}} - P)\widehat{V}_{h+1}^\pi(s,a)\Big| \nonumber\\
        &\leq SA \cdot \Big|(\widehat{P}_{\widehat{\G}} - P_\G)\widehat{V}_{h+1}^\pi(s,a) \Big| + SA \cdot \Big|(P_\G - P)\widehat{V}_{h+1}^\pi(s,a)\Big| \label{eq_terms_pre}\\
        &\leq SA\cdot \sum_{i=h+1}^H i\cdot \epsilon' + SA\cdot \sum_{i=h+1}^H i\cdot \lambda \nonumber\\
        &\leq SAH^2\epsilon' + SAH^2\lambda \label{eq_terms}
    \end{align}
    where $\epsilon'$ is the approximation error between $\widehat{P}_{\widehat{\G}}$ and $P_G$ studied in Lemma \ref{thr:sample_complexity_ctm}, and in the second-to-last step we have used the following derivation
    \begin{align}
        \Big|(\widehat{P}_{\widehat{\G}} - P_\G)\widehat{V}_{h+1}^\pi(s,a) \Big|
        &= \Big|\widehat{P}_{\widehat{\G}}\widehat{V}_{h+1}^\pi(s,a) - P_\G\widehat{V}_{h+1}^\pi(s,a) \Big| \label{recursive_bound}\\
        &= \Big|\sum_{s'}\widehat{P}_{\widehat{\G}}(s'|s,a)\widehat{V}_{h+1}^\pi(s') - \sum_{s'}P_\G(s'|s,a)\widehat{V}_{h+1}^\pi(s') \Big| \nonumber\\
        &= \Big|\sum_{s'}\widehat{P}_{\widehat{\G}}(s'|s,a)\EV_{a' \sim \pi}\Big[ r(s',a') + \widehat{P}_{\widehat{\G}}\widehat{V}_{h+2}^\pi(s',a') \Big] \nonumber \\
        &\quad- \sum_{s'}P_\G(s'|s,a)\EV_{a' \sim \pi}\Big[ r(s',a') + P_\G \widehat{V}_{h+2}^\pi(s',a') \Big] \Big| \nonumber\\
        &= \Big|\sum_{s'}\big(\widehat{P}_{\widehat{\G}}(s'|s,a) - P_\G(s'|s,a)\big)\EV_{a' \sim \pi}\Big[ r(s',a')\Big] \nonumber\\
        &\quad+ \sum_{s'}\widehat{P}_{\widehat{\G}}(s'|s,a)\EV_{a' \sim \pi}\Big[\widehat{P}_{\widehat{\G}}\widehat{V}_{h+2}^\pi(s',a') \Big]
        - \sum_{s'}P_\G(s'|s,a)\EV_{a' \sim \pi}\Big[P_\G  \widehat{V}_{h+2}^\pi(s',a') \Big] \Big| \nonumber\\
        &\leq \epsilon' + \Big| \sum_{s'}\widehat{P}_{\widehat{\G}}(s'|s,a)\EV_{a' \sim \pi}\Big[\widehat{P}_{\widehat{\G}}\widehat{V}_{h+2}^\pi(s',a') \Big]
        - \sum_{s'}P_\G(s'|s,a)\EV_{a' \sim \pi}\Big[P_\G  \widehat{V}_{h+2}^\pi(s',a') \Big] \Big| \nonumber\\
        &= \epsilon' + \Big| \sum_{s'}\widehat{P}_{\widehat{\G}}(s'|s,a)\EV_{a' \sim \pi}\Big[\sum_{s''}\widehat{P}_{\widehat{\G}}(s''|s',a')\EV_{a'' \sim \pi}\Big[r(s'',a'') + \widehat{P}_{\widehat{\G}}\widehat{V}_{h+3}^\pi(s'',a'') \Big] \Big] \nonumber\\
        &\quad- \sum_{s'}P_\G(s'|s,a)\EV_{a' \sim \pi}\Big[\sum_{s''}P_\G(s''|s',a')\EV_{a'' \sim \pi}\Big[r(s'',a'') + P_\G \widehat{V}_{h+3}^\pi(s'',a'') \Big] \Big] \Big| \nonumber\\
        &\leq \epsilon' + \sum_{s',s'',a'} \Big| \widehat{P}_{\widehat{\G}}(s'|s,a)\widehat{P}_{\widehat{\G}}(s''|s',a') - P_\G(s'|s,a)P_\G(s''|s',a') \Big|_1  + \ldots \nonumber\\
        &\leq \epsilon' + \sum_{s',s'',a'}\Big[ \Big|\widehat{P}_{\widehat{\G}}(s'|s,a) - P_\G(s'|s,a)\Big|_1 +  \Big|\widehat{P}_{\widehat{\G}}(s''|s',a') - P_\G(s''|s',a')\Big|_1\Big] + \ldots \nonumber\\
        &\leq \epsilon'  + 2\epsilon' + \ldots \nonumber
    \end{align}
    Hence, due to this recursive unrolling, we have
    \begin{equation*}
        \Big|(\widehat{P}_{\widehat{\G}} - P_\G)\widehat{V}_{h+1}^\pi(s,a) \Big| \leq \sum_{i=h+1}^H i\epsilon' \leq H^2\epsilon.
    \end{equation*}
    Notice that the same argument holds also for the second term of \eqref{eq_terms_pre}, replacing $\epsilon'$ with $\lambda$.

    By plugging the result in equation \eqref{eq_terms} into the initial expression we get
    \begin{align*}
        \Big| \EV_{s \sim P}\Big[\widehat{V}^\pi_1(s) - V_1^\pi(s) \Big] \Big| &\leq \sum_{h=1}^H \EV_{\pi}\Big|(\widehat{P}_{\widehat{\G}} - P)\widehat{V}_{h+1}^\pi(s_h,a_h)\Big|\\
        &\leq \sum_{h=1}^H SAH^2\epsilon' + SAH^2\lambda\\
        &= SAH^3\epsilon' + SAH^3\lambda.
    \end{align*}
\end{proof}

In the following we restate~\citep[][Lemma E.15]{dann2018unifying} for the case of a stationary transition model.
\begin{lemma}
    \label{thr:V_difference}
    For any two MDPs $\mdp'$ and $\mdp''$ with rewards $r'$ and $r''$ and transition models $P'$ and $P''$, the difference in value functions $V', V''$ \wrt the same policy $\pi$ can be written as:
    \begin{equation}
        V'_h(s) - V''_h(s) = \EV_{\mdp'', \pi}\Big[\sum_{i=h}^H [r'(s_i, a_i) - r''(s_i, a_i) + (P'-P'')V'_{i+1}(s_i,a_i)]\mid s_h=s \Big].
    \end{equation}
\end{lemma}

\subsection*{Proofs of Section~\ref{sec:complexity_cuasal_discovery}}

We provide the proof of the sample complexity result for learning the causal structure of a discrete MDP with a generative model.

\sampleComplexityG*
\begin{proof}
    We aim to obtain the number of samples $K$ for which Algorithm~\ref{alg:learning_g} is guaranteed to return a causal structure estimate $\widehat{\G}$ such that $Pr (\widehat{\G} \neq \G_{\epsilon}) \leq \delta$ in a discrete MDP setting.
    First, we can upper bound the probability of the bad event $Pr (\widehat{\G} \neq \G_{\epsilon})$ in terms of the probability of a failure in the independence testing procedure $\ind (X_z, Y_j)$ for a single pair of nodes $X_z \in \G_\epsilon, Y_z \in \G_\epsilon$, \ie
    \begin{equation}
        Pr (\widehat{\G} \neq \G_{\epsilon}) 
        \leq Pr \bigg( \bigcup_{z = 1}^{d_S + d_A} \bigcup_{j = 1}^{d_S} \text{ test } \ind (X_z, Y_j) \text{ fails} \bigg) 
        \leq \sum_{z = 1}^{d_S + d_A} \sum_{j = 1}^{d_S} Pr \bigg(\text{test } \ind (X_z, Y_j) \text{ fails} \bigg), \label{eq:ub}
    \end{equation}
    where we applied an union bound to obtain the last inequality. Now we can look at the probability of a single independence test failure. Especially, for a provably efficient independence test (the existence of such a test is stated by Lemma~\ref{thr:independence_testing}, whereas the Algorithm 2 in~\cite{diakonikolas2021optimal} reports an actual testing procedure), we have $Pr (\text{test } \ind (X_z, Y_j) \text{ fails} ) \leq \delta' $, for any choice of $\delta' \in (0, 1)$, $\epsilon' > 0$, with a number of samples
    \begin{equation}
        K' = C \bigg( \frac{ n \ \log^{1 / 3} (1 / \delta') }{ (\epsilon')^{4 / 3}} + \frac{n \ \log^{1 / 2} (1 / \delta') + \log(1 / \delta')}{(\epsilon')^2} \bigg), \label{eq:K}
    \end{equation}
    where $C$ is a sufficiently large universal constant~\citep[][Theorem 1.3]{diakonikolas2021optimal}. Finally, by letting $\epsilon' = \epsilon$, $\delta' = \frac{\delta}{d_S^2 d_A}$ and combining~\eqref{eq:ub} with~\eqref{eq:K}, we obtain $Pr (\widehat{\G} \neq \G_{\epsilon})$ with a sample complexity
    \begin{equation*}
        K = O \bigg( \frac{n \log(d_S^2 d_A / \delta)}{\epsilon^2} \bigg),
    \end{equation*}
    under the assumption $\epsilon^2 \ll \epsilon^{4 / 3}$, which concludes the proof.
\end{proof}

The proof of the analogous sample complexity result for a tabular MDP setting (Corollary~\ref{thr:sample_complexity_G_tabular}) is a direct consequence of Theorem~\ref{thr:sample_complexity_G} by letting $n = 2, d_S = S, d_A = A$.

\subsection*{Proofs of Section~\ref{sec:complexity_bayesian_network}}

We first report a useful concentration inequality for the L1-distance between the empirical distribution computed over $K$ samples and the true distribution~\citep[][Theorem 2.1]{weissman2003inequalities}.

\begin{lemma}[\citet{weissman2003inequalities}]
\label{thr:l1_concentration}
    Let $X_1, \ldots, X_K$ be i.i.d. random variables over $[n]$ having probabilities $Pr (X_k = i) = P_i$, and let $\widehat{P}_{K} (i) = \frac{1}{K} \sum_{k = 1}^{K} \one (X_k = i)$. Then, for every threshold $\epsilon > 0$, it holds
    \begin{equation*}
        Pr \bigg( \| \widehat{P}_K - P \|_1 \geq \epsilon \bigg) \leq 2 \exp(- K \epsilon^2 / 2n).
    \end{equation*}
\end{lemma}

We can now provide the proof of the sample complexity result for learning the Bayesian network of a discrete MDP with a given causal structure.

\sampleComplexityP*
\begin{proof}
    We aim to obtain the number of samples $K$ for which Algorithm~\ref{alg:learning_p} is guaranteed to return a Bayesian network estimate $\widehat{P}_\G$ such that $ Pr ( \| \widehat{P}_\G - P_\G \|_1 \geq \epsilon ) \leq \delta $ in a discrete MDP setting. First, we note that
    \begin{align}
        Pr \Big( \| \widehat{P}_\G - P_\G \|_1 \geq \epsilon \Big)
        &\leq Pr \bigg( \sum_{j = 1}^{d_S} \| \widehat{P}_j - P_j \|_1 \geq \epsilon \bigg) \label{eq:1} \\
        &\leq Pr \bigg( \frac{1}{d_S} \sum_{j = 1}^{d_S} \| \widehat{P}_j - P_j \|_1 \geq \frac{\epsilon}{d_S} \bigg) \\
        &\leq Pr \bigg( \bigcup_{j = 1}^{d_S} \| \widehat{P}_j - P_j \|_1 \geq \frac{\epsilon}{d_S} \bigg) \label{eq:2} \\
        &\leq \sum_{j = 1}^{d_S} Pr \bigg( \| \widehat{P}_j - P_j \|_1 \geq \frac{\epsilon}{d_S} \bigg), \label{eq:3}
    \end{align}
    in which we employed the property $\| \mu - \nu \|_1 \leq \| \prod_i \mu_i - \prod_i \nu_i \|_1 \leq \sum_{i} \| \mu_i - \nu_i \|_1$ for the L1-distance between product distributions $\mu = \prod_i \mu_i, \nu = \prod_i \nu_i$ to write~\eqref{eq:1}, and we applied a union bound to derive~\eqref{eq:3} from~\eqref{eq:2}. Similarly, we can write
    \begin{align}
        Pr \Big( \| \widehat{P}_j - P_j \|_1 \geq \frac{\epsilon}{d_S} \Big)
        &\leq Pr \bigg( \bigcup_{x \in [n]^{|Z_j|}} \| \widehat{P}_j (\cdot | x) - P_j (\cdot | x) \|_1 \geq \frac{\epsilon}{d_S n^{|Z_j|}} \bigg) \label{eq:4} \\
        &\leq \sum_{x \in [n]^{|Z_j|}} Pr \bigg( \| \widehat{P}_j (\cdot | x) - P_j (\cdot | x) \|_1 \geq \frac{\epsilon}{d_S n^{|Z_j|}} \bigg) \label{eq:5} \\
        &\leq \sum_{x \in [n]^{|Z_j|}} Pr \bigg( \| \widehat{P}_j (\cdot | x) - P_j (\cdot | x) \|_1 \geq \frac{\epsilon}{d_S n^{Z}} \bigg) \label{eq:6}
    \end{align}
    by applying a union bound to derive~\eqref{eq:5} from~\eqref{eq:4}, and by employing Assumption~\ref{ass:sparseness} to bound $|Z_j|$ with $Z$ in~\eqref{eq:6}. We can now invoke Lemma~\ref{thr:l1_concentration} to obtain the sample complexity $K'$ that guarantees $ Pr ( \| \widehat{P}_j (\cdot | x) - P_j (\cdot | x) \|_1 \geq \epsilon' ) \leq \delta' $, \ie
    \begin{equation*}
        K' = \frac{2 n \log (2 / \delta')}{ (\epsilon')^2} = \frac{2 \ d_S^2 \ n^{2Z + 1} \ \log (2 d_S n^Z / \delta)}{\epsilon^2 },
    \end{equation*}
    where we let $\epsilon' = \frac{\epsilon}{d_S n^Z} $, $\delta' = \frac{\delta}{d_S n^Z}$. Finally, by summing $K'$ for any $x \in [n]^{|Z_j|}$ and any $j \in [d_S]$, we obtain
    \begin{equation*}
        K = \sum_{j \in [d_S]} \sum_{x \in [n]^{|Z_j|}} K' \leq \frac{2 \ d_S^3 \ n^{3Z + 1} \ \log (2 d_S n^Z / \delta)}{\epsilon^2 },
    \end{equation*}
    which proves the theorem.
\end{proof}

To prove the analogous sample complexity result for a tabular MDP we can exploit a slightly tighter concentration on the KL divergence between the empirical distribution and the true distribution in the case of binary variables~~\citep[][Theorem 2.2.3]{dembo2009ldp}\footnote{Also reported in \citep[][Example 1]{mardia2020concentration}.}, which we report for convenience in the following lemma.

\begin{lemma}[\citet{dembo2009ldp}]
\label{thr:KL_binary_concentration}
    Let $X_1, \ldots, X_K$ be i.i.d. random variables over $[2]$ having probabilities $Pr (X_k = i) = P_i$, and let $\widehat{P}_{K} (i) = \frac{1}{K} \sum_{k = 1}^{K} \one (X_k = i)$. Then, for every threshold $\epsilon > 0$, it holds
    \begin{equation*}
        Pr \bigg( d_{KL} \big( \widehat{P}_K || P \big) \geq \epsilon \bigg) \leq 2 \exp(- K \epsilon).
    \end{equation*}
\end{lemma}

We can now provide the proof of Corollary~\ref{thr:sample_complexity_P_tabular}.

\sampleComplexityPtabular*
\begin{proof}
    We aim to obtain the number of samples $K$ for which Algorithm~\ref{alg:learning_p} is guaranteed to return a Bayesian network estimate $\widehat{P}_\G$ such that $ Pr ( \| \widehat{P}_\G - P_\G \|_1 \geq \epsilon ) \leq \delta $ in a tabular MDP setting. We start by considering the KL divergence $d_{KL} ( \widehat{P}_\G || P_\G )$. Especially, we note
    \begin{align*}
        d_{KL} \big( \widehat{P}_\G || P_\G \big) 
        &= \sum_{X, Y}  \widehat{P}_\G (X, Y) \log \frac{\widehat{P}_\G (X, Y)}{P_\G (X, Y)} \\
        &= \sum_{X, Y}  \widehat{P}_\G (X, Y) \log \frac{ \prod_{j = 1}^{S} \widehat{P}_j (Y[j] | X[Z_j])}{\prod_{j = 1}^{S} P_j (Y[j] | X[Z_j])} \\
        &= \sum_{X, Y}  \widehat{P}_\G (X, Y) \sum_{j = 1}^{S} \log \frac{ \widehat{P}_j (Y[j] | X[Z_j])}{ P_j (Y[j] | X[Z_j])} = \sum_{j = 1}^{S} d_{KL} \big( \widehat{P}_j || P_j \big).
    \end{align*}
    Then, for any $\epsilon' > 0$ we can write
    \begin{align}
         Pr \Big( d_{KL} \big( \widehat{P}_\G || P_\G \big) \geq \epsilon' \Big) 
         &\leq Pr \bigg( \bigcup_{j = 1}^{S} d_{KL} \big( \widehat{P}_j || P_j \big) \geq \frac{\epsilon'}{S} \bigg) \label{eq:7} \\
         &\leq \sum_{j = 1}^{S} Pr \bigg( d_{KL} \big( \widehat{P}_j || P_j \big) \geq \frac{\epsilon'}{S} \bigg) \label{eq:8} \\
         &\leq \sum_{j = 1}^{S} Pr \bigg( \bigcup_{x \in [2]^{|Z_j|}} d_{KL} \big( \widehat{P}_j (\cdot | x) || P_j (\cdot | x) \big) \geq \frac{\epsilon'}{S 2^{|Z_j|}} \bigg) \label{eq:9} \\
         &\leq \sum_{j = 1}^{S} \sum_{x \in [2]^{|Z_j|}} Pr \bigg( d_{KL} \big( \widehat{P}_j (\cdot | x) || P_j (\cdot | x) \big) \geq \frac{\epsilon'}{S 2^{|Z_j|}} \bigg) \label{eq:10} \\
         &\leq \sum_{j = 1}^{S} \sum_{x \in [2]^{|Z_j|}} Pr \bigg( d_{KL} \big( \widehat{P}_j (\cdot | x) || P_j (\cdot | x) \big) \geq \frac{\epsilon'}{S 2^{Z}} \bigg), \label{eq:11}
    \end{align}
    in which we applied a first union bound to get~\eqref{eq:8} from~\eqref{eq:7}, a second union bound to get~\eqref{eq:10} from~\eqref{eq:9}, and Assumption~\ref{ass:sparseness} to bound $|Z_j|$ with $Z$ in~\eqref{eq:11}. We can now invoke Lemma~\ref{thr:KL_binary_concentration} to obtain the sample complexity $K''$ that guarantees $ Pr ( d_{KL} ( \widehat{P}_j (\cdot | x) || P_j (\cdot | x) ) \geq \epsilon'' ) \leq \delta'' $, \ie
    \begin{equation*}
        K'' = \frac{ \log(2 / \delta'')}{\epsilon''} = \frac{ S 2^Z \log(2 S 2^Z / \delta')}{ \epsilon'},
    \end{equation*}
    where we let $\epsilon'' = \frac{\epsilon'}{S 2^Z}$, and $\delta'' = \frac{\delta'}{S 2^Z}$ for any choice of $\delta' \in (0, 1)$. By summing $K''$ for any $x \in [2]^{|Z_j|}$ and and $j \in [S]$, we obtain the sample complexity $K'$ that guarantees $ Pr ( d_{KL} ( \widehat{P}_\G || P_\G ) \geq \epsilon' ) \leq \delta' $, \ie
    \begin{equation}
        K' = \sum_{j = 1}^{S} \sum_{x \in [2]^{|Z_j|}} K'' \leq \frac{ S^2 2^{2Z} \log(2 S 2^Z / \delta')}{ \epsilon'}. \label{eq:12}
    \end{equation}
    Finally, we employ the Pinsker's inequality $\| \widehat{P}_\G - P_\G \|_1 \leq \sqrt{ 2 d_{KL} ( \widehat{P}_\G || P_\G) }$~\cite{csiszar1967information} to write
    \begin{equation*}
        Pr \Big( d_{KL} ( \widehat{P}_\G || P_\G) \geq \epsilon' \Big) 
        = Pr \Big( \sqrt{ 2 d_{KL} ( \widehat{P}_\G || P_\G)} \geq \sqrt{ 2 \epsilon'} \Big) 
        \geq Pr \Big( \| \widehat{P}_\G - P_\G \|_1 \geq \sqrt{ 2 \epsilon'} \Big),
    \end{equation*}
    which gives the sample complexity $K$ that guarantees $ Pr ( \| \widehat{P}_\G - P_\G \|_1 \geq \epsilon ) \leq \delta $ by letting $\epsilon' = \frac{\epsilon^2}{2}$ and $\delta' = \delta$ in~\eqref{eq:12}, \ie
    \begin{equation*}
        K = \frac{ 2 S^2 2^{2Z} \log( 2 S 2^Z / \delta)}{ \epsilon^2 }.
    \end{equation*}
\end{proof}

%% file: appendix_numerical_validation.tex
Let $A$, $W$, $P$, $S$, $D$, $C$ and $St$ represent academic performance, weight, physical activity, sleep, diet and study respectively. We start by defining the causal transition model as follows:

\begin{align*}
    A &\sim \mathcal{N}(\mu_A, \sigma)
    \\
    W &\sim \mathcal{N}(\mu_W, \sigma)
    \\
    \mu_A &= A + 0.2 D + 0.5 S + 0.8 St - 0.8
    \\
    \mu_W &= W - 0.5 D - 0.5 P + 1
    \\
    \sigma &= 0.1
\end{align*}

The transition model for a specific environment is then generated by adding independent white noise to each coefficient in the above equations for $\mu_A$ and $\mu_W$, including the ones not shown because set to 0:

\begin{align*}
    \mu_A &= (1+\epsilon_1) A + \epsilon_2 W + \epsilon_3 P + (0.2 + \epsilon_4) D + (0.5 + \epsilon_5) S + \epsilon_6 C + (0.8+\epsilon_7) St - 0.8
    \\
    \mu_W &= \epsilon_8 A + (1+\epsilon_9) W + (-0.5 + \epsilon_{10}) P + (-0.5 + \epsilon_{11}) D + \epsilon_{12} S + \epsilon_{13} C + \epsilon_{14} St + 0.8
    \\
    \epsilon_i &\sim \mathcal{N}(0, 0.1) \quad \forall i \in \{1\ldots 14\}
\end{align*}

\begin{figure}[H]
    \centering
    \includegraphics[width=0.5\linewidth]{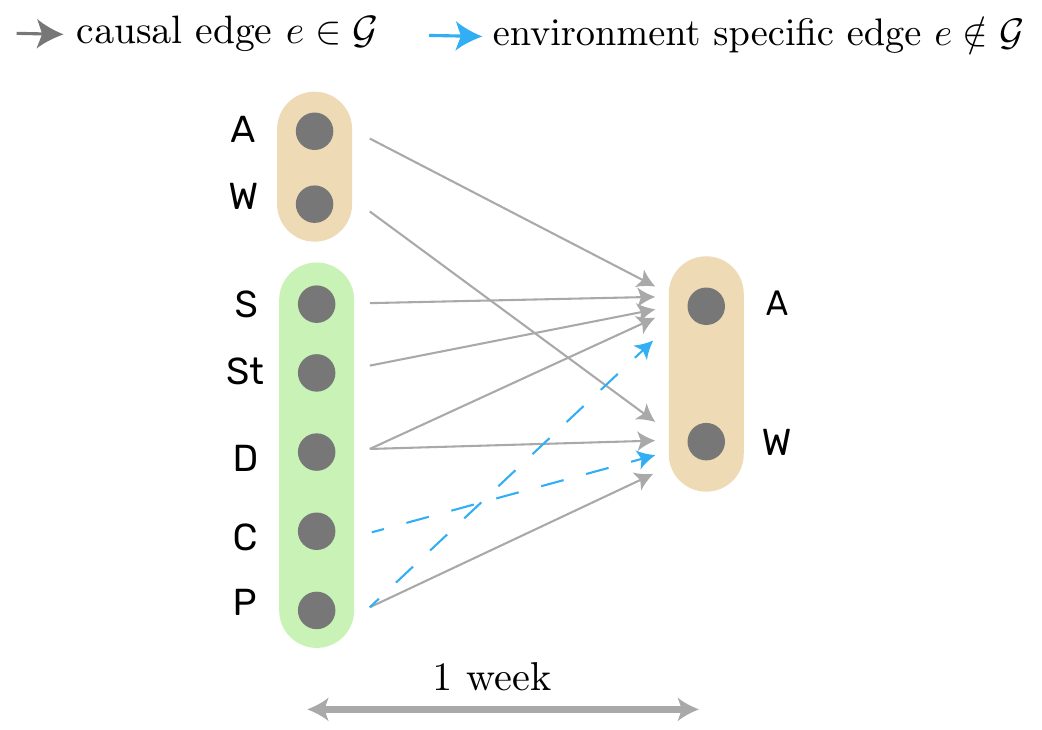}
    \caption{Bayesian network for a specific environment (person) of our synthetic example. Grey edges are causal and therefore shared by other environments in the same universe $\uni$, while blue dashed edges are environment-specific dependencies. }
    \label{fig:illustrative_ex}
\end{figure}